\documentclass{article}

     \usepackage[final]{neurips_2020}

\usepackage[utf8]{inputenc} 
\usepackage[T1]{fontenc}    
\usepackage{hyperref}       
\usepackage{url}            
\usepackage{booktabs}       
\usepackage{amsfonts}       
\usepackage{nicefrac}       
\usepackage{microtype}      
\usepackage{xcolor}
\usepackage{soul}

\title{On the linearity of large non-linear models: when and why the tangent kernel is constant
}

\author{%
  Chaoyue Liu\thanks{Department of Computer Science, The Ohio State University. E-mail: \texttt{liu.2656@osu.edu}} 
  \And Libin Zhu\thanks{Department of Computer Science and Halicioğlu Data Science Institute, University of California, San Diego. E-mail: \texttt{l5zhu@ucsd.edu}} 
  \And Mikhail Belkin\thanks{Halicioğlu Data Science Institute, University of California, San Diego. E-mail: \texttt{mbelkin@ucsd.edu}}
  \\
}
\usepackage[citestyle=numeric,natbib=true,bibstyle=numeric,maxbibnames=99]{biblatex}
\usepackage{enumitem}
\usepackage{bbm}
\usepackage{amsthm}
\usepackage{graphicx}
\usepackage{subcaption}
\usepackage{amsmath}
\usepackage{amssymb}
\usepackage{wrapfig}

\renewcommand\L{{\mathcal{L}}}

\newcommand\R{{\mathbb{R}}}

\renewcommand\ln{{\mathrm{ln}}}
\renewcommand\P{{\mathbb{P}}}

\def\rva{{\mathbf{a}}}
\def\rvb{{\mathbf{b}}}
\def\rvc{{\mathbf{c}}}

\def\rvs{{\mathbf{s}}}
\def\rvt{{\mathbf{t}}}

\def\rvv{{\mathbf{v}}}
\def\rvw{{\mathbf{w}}}
\def\rvx{{\mathbf{x}}}
\def\rvy{{\mathbf{y}}}
\def\rvz{{\mathbf{z}}}

\def\rmT{{\mathbf{T}}}

\def\rmW{{\mathbf{W}}}

\newtheorem{lemma}{Lemma}[section]
\newtheorem{thm}{Theorem}[section]

\newtheorem{prop}{Proposition}[section]

\theoremstyle{definition}

\newtheorem{defi}{Definition}[section]

\newtheorem{remark}{Remark}[section]
\newtheorem{assumption}{Assumption}[section]

\begin{document}

\maketitle

\begin{abstract} 

The goal of this work is to shed light on the remarkable phenomenon of transition to linearity of certain neural networks as their width approaches infinity. We show that the transition to linearity of the model and, equivalently,  constancy of the (neural) tangent kernel (NTK) result from the scaling properties of the norm of the Hessian matrix of the network as a function of the network width.
We present a general framework for understanding the {constancy of the tangent kernel via} Hessian scaling applicable to the standard classes of neural networks. Our analysis provides a new perspective on the phenomenon of constant tangent kernel, which is different from the widely accepted ``lazy training''.
Furthermore, we show that the transition to linearity is not a general property of wide neural networks and does not hold when the last layer of the network is non-linear. 
It is also not necessary for successful optimization by gradient descent. 

\end{abstract}
\section{Introduction}

As the width of certain non-linear neural networks increases, they become linear  functions of their parameters. 
This remarkable property of large models was first identified in~\cite{jacot2018neural} where it was stated in terms of the constancy of the (neural) tangent kernel during the training process. More precisely, consider a neural network or,  generally, a machine learning model $f(\rvw;\rvx)$, which takes $\rvx$ as input and has $\rvw$ as its (trainable) parameters.
Its tangent kernel $K_{(\rvx,\rvz)}(\rvw)$ is defined as follows:
\begin{equation}\label{eq:ntk}
    K_{(\rvx,\rvz)}(\rvw) := \nabla_\rvw f(\rvw;\rvx)^T\nabla_\rvw f(\rvw;\rvz), \quad \textrm{for fixed inputs }\rvx, \rvz \in \mathbb{R}^{d}.
\end{equation}
The key finding of~\cite{jacot2018neural} was the fact that for some wide neural networks the kernel $K_{(\rvx,\rvz)}(\rvw)$ is a constant function of the weight $\rvw$ during training. While in the literature, including~\cite{jacot2018neural}, this phenomenon is  described in terms of the (linear) training dynamics, it is important to note that the tangent kernel is associated to the model itself.  As such, it does not depend on the optimization algorithm or the choice of a loss function, which are parts of the training process.

The goal of this work is to clarify a number of issues related to the constancy of the tangent kernel, to provide specific conditions when the kernel is constant, i.e., when non-linear models in the limit, as their width approach infinity, become linear,  and also to explicate the regimes when they do not. One important conclusion of our analysis is that the ``transition to linearity'' phenomenon   discussed in this work  (equivalent to constancy of tangent kernel) cannot be explained by ``lazy training''~\cite{chizat2019lazy} associated to small change of parameters from the initialization point or model rescaling,  which is widely held to be the reason for constancy of the tangent kernel, e.g., ~\cite{sun2019optimization,amari2020any,ghorbani2019linearized} (see Section~\ref{sec:comments} for a detailed discussion).   The transition to linearity is neither  due to a choice of a scaling of the model, nor is  a universal property of large models including infinitely wide neural networks. 
In particular, the models shown  to transition to linearity in this paper  become linear in a Euclidean ball of an arbitrary fixed radius, not just in a small vicinity of the initialization point, where higher order terms of the Taylor series can be ignored.

Our first observation\footnote{While it is a known mathematical fact (see~\cite{868044,sakai1996riemannian}),  we have not seen it in the neural network literature, possibly due to  the discussion  typically  concerned with the  dynamics  of optimization. As  a special case, note that while $\nabla f \equiv const$ clearly implies that $f$ is linear, it is not a priori obvious that the weaker condition  $\|\nabla f\| \equiv const$ is also sufficient.
} is that a function $f(\rvw,\rvx)$ has a constant tangent kernel {\it if and only if} it is linear in $\rvw$, that is
$$
f(\rvw,\rvx) = \rvw^T \phi(\rvx) + f_0(\rvx)
$$
for some ``feature map'' $\phi$ and function $f_0$. 
Thus the constancy of the tangent kernel is directly  linked to the linearity of the underlying model. 

So what is the underlying reason that some large models transition to linearity as a function of the  parameters and when do we expect it to be the case?
As known from the mathematical analysis, the deviation from the linearity is controlled by the second derivative, which is represented, for a multivariate function $f$, by the {\it Hessian matrix}  $H$. If its spectral norm $\|H\|$  is small compared to the gradient $\nabla_{\rvw} f$ in a ball of a certain radius, the function $f$ will be close to linear and will have near-constant tangent kernel in that ball. Crucially, the spectral norm $\|H\|$ depends not just on the magnitude of its entries, but also on the structure of the matrix $H$.
This simple idea underlies the analysis in this paper. Note that throughout this paper we consider the Hessian of the model $f$, not of any related loss function.

{\bf Constant tangent kernel for neural networks with linear output layer.} In what follows we analyze  the class of neural networks with linear output layer, which includes networks that have been found to have constant tangent kernel in~\cite{jacot2018neural,lee2019wide,du2018gradientshallow} and other works. We show that while the gradient norm   $\|\nabla_{\rvw} f\|$ is (omitting log factors) of the order $\Theta(1)$ w.r.t. the network width $m$, the spectral norm of the Hessian matrix $\|H\|$ scales with $m$  as ${1/\sqrt{m}}$.
In the infinite width limit, this implies a vanishing Hessian and hence transition to linearity of the model in a ball of an arbitrary fixed radius.
A consequence of this analysis is the  constancy of the tangent kernel, providing a different perspective on the results in~\cite{jacot2018neural} and the follow-up works. 

We proceed to expose the underlying  reason why the Hessian matrix scales differently from the gradient and delimit the regimes where this phenomenon exists. {
As we show, the scaling of the Hessian spectral norm is controlled by both the {\it $\infty$-norms} of the vectors $\partial f / \partial \alpha^{(l)}, l \in [L]$, where $\alpha^{(l)}$ is the (vector) value of the $l$-th hidden layer, and the norms of layer-wise derivatives (specifically, the $(2,1,1)$-norm of the corresponding order $3$ tensors). On the other hand, the scaling of the gradient and the tangent kernel is controlled by the {\it $2$-norms} (i.e., Euclidean norms) of $\partial f / \partial \alpha^{(l)}$. As the network width $m$ (i.e., minimal width of hidden layers) is sufficiently large, the discrepancy between the  the $\infty$-norm and $2$-norm increases, while the  $(2,1,1)$-norms remain of the same order. Hence we obtain the discrepancy between the scaling behaviors of the Hessian and gradient.
}



{\bf Non-constancy of tangent kernels.} We proceed to demonstrate, both theoretically (Section~\ref{sec:non_constant}) and experimentally (Section~\ref{subsec:numeric}), that the  constancy of tangent kernel is not a general property of large models, including wide networks, even in the ``lazy'' training regime. In particular, 
if the output layer of a network is nonlinear, {e.g., if there is a non-linear activation on the output,} the Hessian norm does not tend to zero as $m\to\infty$, and constancy of tangent kernel will not hold in any fixed neighborhood and along the optimization path, although each individual parameter may undergo only a small change. This demonstrates that the constancy of the tangent kernel relies on specific  structural properties of the models.
Similarly, we show that inserting a narrow ``bottleneck'' layer, even if it is linear,  will generally result in the loss of near-linearity, as the Hessian norm becomes large compared to the gradient $\nabla_{\rvw} f$ of the model.

Importantly, as we discuss in Section~\ref{sec:optimization}, non-constancy of the tangent kernel does not preclude efficient optimization. We construct examples of wide networks which can be provably optimized by gradient descent, yet with tangent kernel provably far from constant along the optimization path and with Hessian norm $\Omega(1)$, same as the gradient.

\subsection{Discussion and related work}\label{sec:comments}
We proceed to make a number of remarks in the context of some recent work on the subject. 

{\bf Is the weight change from the initialization to convergence small?} In the recent literature (e.g.,\cite{sun2019optimization,amari2020any,ghorbani2019linearized}) it is sometimes asserted that the constancy of tangent kernel is explained by small change of weight vector  during training, a property related to  ``lazy training'' introduced in~\cite{chizat2019lazy}. 
It is important to point out that the notion  ``small''  depends crucially on the measurement. Indeed, as we discuss below, when measured correctly in relation to the tangent kernel, the change from initialization is {\it not} small. 

Let $\rvw_0$ and $\rvw^*$ be the weight vectors at initialization and at convergence respectively. For example, consider a one hidden layer network of width $m$. Each {\it component} of the weight vector is updated by $O(1/\sqrt{m})$ under gradient descent, as shown in~\cite{jacot2018neural}, and hence for wide networks $\|\rvw^* - \rvw_0\|_\infty = O(1/\sqrt{m})$, a quantity that vanishes with the increasing width. In contrast, the change of the  {\it Euclidean norm}  is not small in training, 
$
\|\rvw^* - \rvw_0\|^2 = \sum_{i=1}^m (w^{*}_i - w_{0,i})^2=\Theta(1)
$.
Thus convergence happens within a Euclidean  ball with  radius independent of the network width. 

In fact, the Euclidean norm of the change of the weight vector cannot be small for Lipschitz continuous models, even in the limit of infinite parameters. This is because \begin{equation}\label{eq:norm_lower_bound}
    \|\rvw^*-\rvw_0\| \ge \frac{|f(\rvw_0;\rvx)-y|}{\sup_{\rvw}\|\nabla_{\rvw} f(\rvw)\|}
\end{equation}
where is $y$ is the label at $\rvx$. 
Note that the difference $|f(\rvw_0;\rvx)-y|$, between the initial prediction $f(\rvw_0;\rvx)$ and the ground truth label $y$, is of the same order as $\| \nabla_{\rvw} f\|$. Thus, we see that $\|\rvw^*-\rvw_0\| = \Omega(1)$, no matter how many parameters the model $f$ has.

We note that the (approximate) linearity of a model in a certain region (and hence the constancy of the tangent kernel) is predicated on the second-order term of the Taylor expansion $(\rvw-\rvw_0)^TH(\rvw-\rvw_0)$. That term depends on the {\it Euclidean distance} from the initialization $\|\rvw-\rvw_0\|$ (and the spectral norm of the Hessian), instead of the $\infty$-norm $\|\rvw-\rvw_0\|_\infty$. Since, as we discussed above, these norms are different by a factor of $\sqrt{m}$, an argument based on small change of individual parameters from initialization cannot explain the remarkable phenomenon of constant tangent kernel.

In contrast to these interpretations, we show that certain large networks have near constant tangent kernel in a ball of fixed radius due to the vanishing Hessian norm, as their widths approach infinity. Indeed, that is the case for networks analyzed in the NTK literature~\cite{jacot2018neural,lee2019wide,du2018gradientshallow,du2018gradientdeep}. 
 
{\bf Can the transition to linearity be explained by model rescaling?}
The work~\cite{chizat2019lazy} introduced the term ``lazy training'' and proposed a mechanism for the constancy of the tangent kernel based on rescaling the model. While, as shown in~\cite{chizat2019lazy}, model rescaling can lead to lazy training, as we discuss below, it does not explain the phenomenon of constant tangent kernel in the setting of the original paper~\cite{jacot2018neural} and consequent works. 

{

Specifically,~\cite{chizat2019lazy} provides the following  criterion  for the near constancy of the tangent kernel (using their notation):
\begin{equation}\label{eq:criterion}
    \kappa_{f}(\rvw_0):= \underbrace{\left\|f(\rvw_0)-y\right\|}_{\mathcal{A}}\underbrace{\frac{\|D^2f(\rvw_0)\|}{\|Df(\rvw_0)\|^2}}_{\mathcal{B}} \ll 1,
\end{equation}
Here $y$ is the ground truth label, $D^2f(\rvw_0)$ is the Hessian of the model $f$ at initialization and $\|Df(\rvw_0)\|^2$ is the norm of the gradient, i.e., a  diagonal entry of the tangent kernel.

The paper \cite{chizat2019lazy} shows that the model $f$ can be rescaled to satisfy the condition in Eq.(\ref{eq:criterion}) as follows. Consider a  rescaled model $\alpha f$ with a scaling factor $\alpha \in \R, \alpha >0$. Then the quantity $\kappa_{\alpha f}$ becomes 
\begin{equation}\label{eq:rescale_criterion}
    \kappa_{\alpha f}(\rvw_0)= \frac{1}{\alpha}\left\|\alpha f(\rvw_0)-y\right\|\frac{\|D^2  f(\rvw_0)\|}{\|D f(\rvw_0)\|^2}= \left\| f(\rvw_0)-\frac{y}{\alpha}\right\|\frac{\|D^2  f(\rvw_0)\|}{\|Df(\rvw_0)\|^2}.
\end{equation}
Assuming that $f(\rvw_0) = 0$ and 
 choosing a large $\alpha$,  forces  $\kappa_{\alpha f}\ll 1$, by rescaling the factor $\mathcal{A}$ to be small, {while keeping $\mathcal{B}$ unchanged}.

While rescaling the model, together with  the important assumption of $f(\rvw_0) = 0$, leads to a lazy training regime, we point out that it is not the same regime as observed in the original work~\cite{jacot2018neural} and followup papers such as~\cite{lee2019wide,du2018gradientshallow} and  also different from practical neural network training, since we usually have  $\mathcal{A} = \|f(\rvw_0)-y\|= \Theta(1)$ in these settings. Specifically:
\begin{itemize}
    \item The assumption of $f(\rvw_0)=0$ is necessary for the rescaled models in~\cite{chizat2019lazy} to have $\mathcal{A}\ll 1$.
Yet, the networks, such as those analyzed in \cite{jacot2018neural}, are initialized so that $f(\rvw_0)=\Theta(1)$. 
\item From Eq.(\ref{eq:rescale_criterion}), we see that rescaling the model $f$ by $\alpha$ is equivalent to rescaling the ground truth label $y$ by $1/\alpha$ without changing the model (this can also be seen from the loss function, cf. Eq.(2)  of \cite{chizat2019lazy}). When $\alpha$ is large, the rescaled label  $y/\alpha$ is close to zero. However,  no such rescaling happens in practice or in works, such as~\cite{jacot2018neural, lee2019wide,du2018gradientshallow}.
The training dynamics of the model with the label $y/\alpha$ does not generally  match the dynamics of the original problem with the label $y$ and will result in a different solution.
\end{itemize}
Since $\mathcal{A} = \Theta(1)$, in the NTK setting and many practical settings, to satisfy the criterion in Eq.(\ref{eq:criterion}), the model needs to have  
$\mathcal{B}={\|D^2f(\rvw_0)\|}/{\|Df(\rvw_0)\|^2} \ll 1$. 
In fact, we note that the analysis of 2-layer networks in~\cite{chizat2019lazy} uses a different argument, not based on model rescaling.
Indeed, as we  show in this work,  $\mathcal{B}$ is small for 
a broad class of wide neural networks with linear output layer, due to a vanishing norm of the Hessian as the width of the network increases.


In summary, the rescaled models satisfy the criterion, $\kappa\ll 1$, by scaling the factor $\mathcal{A}$ to be small, while the neural networks,  such as the ones considered in the original work~\cite{jacot2018neural}, satisfy this criterion by having $\mathcal{B} \ll 1$, while $\mathcal{A}=\Theta(1)$.}

{\bf Is near-linearity necessary for optimization?} 
In this work we  concentrate on understanding the phenomenon of constant tangent kernel, when  large non-linear systems transition to linearity with  increasing number of parameters. The linearity implies convergence of gradient descent assuming that the tangent kernel is non-degenerate at initialization.  However, it is important to emphasize that the linearity or near-linearity is not a necessary condition for convergence. Instead, convergence is implied by uniform conditioning of the tangent kernel in a neighborhood of a certain radius, while the linearity is controlled by the norm of the Hessian. These are conceptually and practically different phenomena as we show on an example of a wide shallow network  with a  non-linear output layer in Section~\ref{sec:optimization}. See also our paper~\cite{liu2020toward} for an in-depth discussion of optimization.

\section{Notation and Basic Results on Tangent Kernel and Hessian}
\subsection{Notation {\bf and Preliminary}}
We use bold lowercase letters, e.g., $\rvv$, to denote vectors,  capital letters, e.g., $W$, to denote matrices, and bold capital letters, e.g., $\rmW$, to denote matrix tuples  or higher order tensors.
We denote the set $\{1,2,\cdots, n\}$ as $[n]$.  
We use the following norms in our analysis:
For vectors, we use $\|\cdot \|$ to denote the Euclidean norm (a.k.a. vector $2$-norm) and $\|\cdot\|_{\infty}$ for the ${\infty}$-norm;   For matrices, we use $\|\cdot\|$ to denote the spectral norm (i.e., matrix 2-norm) and $\|\cdot\|_F$ to denote the Frobenius norm. In addition, we use tilde, e.g., $\tilde{O}(\cdot)$, to suppress  logarithmic terms in Big-O notation.

We use $\nabla_{\rvw} f$ to represent the  derivative of $f(\rvw;\rvx)$  with respect to  $\rvw$. For (vector-valued) functions, we use the following definition of its Lipschitz continuity:
\begin{defi}
 A function $f:\mathbb{R}^m \to \mathbb{R}^n$ is called $\mathsf{L}_f$-Lipschitz continuous, if there exists $\mathsf{L}_f>0$, such that for all $\rvx,\rvz\in\mathbb{R}^m$, $\|f(\rvx)-f(\rvz)\| \le \mathsf{L}_f \|\rvx-\rvz\|$.
\end{defi}

For an order $3$ tensor, we define its $(2,2,1)$-norm:

\begin{defi}[$(2,2,1)$-norm of order $3$ tensors]
 For an order $3$ tensor $\rmT\in \mathbb{R}^{d_1\times d_2 \times d_3}$, with components $T_{ijk}, i\in [d_1], j\in [d_2], k\in [d_3]$, define its ${(2,2,1)}$-norm as 
\begin{equation}\label{eq:defi_order_3_norm}
    \|\rmT\|_{2,2,1} := \sup_{\|\rvx\|=\|\rvz\|=1} \sum_{k=1}^{d_3}\Big|\sum_{i=1}^{d_1}\sum_{j=1}^{d_2}T_{ijk}x_iz_j\Big|, \quad \textrm{where } \ \rvx \in \mathbb{R}^{d_1},\rvz \in \mathbb{R}^{d_2}.
\end{equation}
\end{defi}
We will later need the following proposition which is essentially a special case of the the Holder inequality. 
\begin{prop}\label{prop:221}
Consider a matrix $A$ with components $A_{ij} = \sum_k T_{ijk}v_k$, where $T_{ijk}$ is a component of the order $3$ tensor $\rmT$ and $v_k$ is a component of vector $\rvv$. Then the spectral norm of $A$ satisfies 
\begin{equation}\|A\| \le \|\rmT\|_{2,2,1}\|\rvv\|_{\infty}
\end{equation}
\end{prop}
\begin{proof}
Note that spectral norm  is defined as $\|A\| = \sup_{\|\rvx\|=\|\rvz\|=1}\rvx^TA\rvz$. Then
\begin{equation}
    \|A\|=\sup_{\|\rvx\|=\|\rvz\|=1}\sum_{i,j,k} T_{ijk}x_iz_jv_k\le \max_{k}|v_k|\sup_{\|\rvx\|=\|\rvz\|=1}\sum_{k}\Big|\sum_{i,j} T_{ijk}x_iz_j\Big| = \|\rvv\|_{\infty}\|\rmT\|_{2,2,1}.\nonumber
\end{equation}
\end{proof}


\subsection{Tangent kernel and the Hessian}
Consider a machine learning model, e.g., a neural network, $f(\rvw;\rvx)$, which takes $\rvx\in \mathbb{R}^d$ as input and has $\rvw\in\mathbb{R}^p$ as the trainable parameters. Throughout this paper, we assume $f$ is twice differentiable  with respect to the parameters $\rvw$. To simplify the analysis, we further assume the output of the model $f$ is a scalar.
Given a set of points 
$\{\rvx_i\}_{i=1}^n$, where each $\rvx_i\in \mathbb{R}^d$, one can build a $n\times n$ tangent kernel matrix $K(\rvw)$, where each entry $K_{ij}(\rvw) = K_{(\rvx_i,\rvx_j)}(\rvw)$.

As discovered in~\cite{jacot2018neural} and analyzed in the consequent works~\cite{lee2019wide,du2018gradientshallow} the tangent kernel is constant for certain infinitely wide networks during training by gradient descent methods.
First, we observe that the constancy of the tangent kernel is equivalent to the linearity of the model. While the mathematical result is not new (see~\cite{868044,sakai1996riemannian}), we have not seen this stated in the machine learning literature (the proof can be found in Appendix~\ref{secapp:linearity}).
\begin{prop}[Constant tangent kernel = Linear model]\label{prop:linearity}
The tangent kernel of a differentiable function $f(\rvw;\rvx)$ is constant if and only if $f(\rvw;\rvx)$ is linear in $\rvw$.
\end{prop}
Of course for a model to be linear it is necessary and sufficient for the Hessian to vanish.
The following proposition extends this result by showing that small Hessian norm is a sufficient condition for near-constant tangent kernel. The proof can be found in Appendix~\ref{secapp:pfhessian}.
\begin{prop}[Small Hessian norm $\Rightarrow$ Small change of tangent kernel]\label{prop:Hessiantokernel}
Given a point $\rvw_0\in\mathbb{R}^p$ and a ball $B(\rvw_0,R) := \{\rvw\in \mathbb{R}^p: \|\rvw-\rvw_0 \|\le R\}$ with fixed radius $R>0$, if the Hessian matrix satisfies $\|H(\rvw)\| < \epsilon$, where $\epsilon>0$, for all $\rvw\in B(\rvw_0,R)$,  then the tangent kernel $K(\rvw)$ of the model, as a function of $\rvw$, satisfies
\begin{equation}
   | K_{(\rvx,\rvz)}(\rvw) - K_{(\rvx,\rvz)}(\rvw_0)|  = O(\epsilon R), \quad \forall \rvw\in B(\rvw_0,R),\ \forall \rvx,\rvz \in \mathbb{R}^d. 
\end{equation}
\end{prop}

As we shall see in Section~\ref{sec:linear_ouput_nn}, all neural networks that are proven in \cite{jacot2018neural,du2018gradientshallow,du2018gradientdeep} to have (near) constant tangent kernel during training, have small (zero, in the limit of $m\to\infty$) spectral norms of the corresponding Hessian matrices.



\section{ Transition to linearity: non-linear neural networks with linear output layer}\label{sec:linear_ouput_nn}
In this section, we analyze the class of neural networks with linear output layer, i.e., there is no non-linear activation on the final output. 
We show that the spectral norm of the Hessian matrix becomes small, when the width of each hidden layer increases. In the limit of infinite width, these spectral norms vanish and the models become linear, with constant tangent kernels. We point out that the neural networks that are already shown to have constant tangent kernels in~\cite{jacot2018neural,lee2019wide,du2018gradientshallow} fall in this category.

\subsection{$1$-hidden layer neural networks}\label{sec:warmup}

As a warm-up for the more complex setting of deep networks, we start by considering the simple case of a shallow fully-connected neural network  with a fixed output layer, defined as follows:
\begin{equation}\label{eq:defi1layer}
    f(\rvw;x) = \frac{1}{\sqrt{m}}\sum_{i=1}^m v_i \alpha_i(x),\textrm{with~~} \alpha_i(x)= \sigma (w_ix),~~ x \in \mathbb{R}.
\end{equation}
Here $m$ is the number of  neurons in the hidden layer, $\rvv = (v_1, \cdots, v_m)$ is the vector of output layer weights, $\rvw = (w_1,\cdots,w_m)\in \mathbb{R}^m$ is the weights in the hidden layer. We assume that the activation function $\sigma(\cdot)$ is  $\beta_{\sigma}$-smooth {(e.g., $\sigma$ can be $sigmoid$ or $tanh$)}. We initialize at random, $w_i\sim \mathcal{N}(0,1)$ and $v_i \in \{-1,1\}$. We treat $\rvv$ as fixed parameters and $\rvw$ as trainable parameters. For the purpose of illustration, we assume the input $x$ is of dimension $1$, and the multi-dimensional analysis is similar.
\begin{remark}
This definition of a shallow neural network (i.e., with the presence of a factor $1/\sqrt{m}$ and $v_i$ and $w_i$ of order $O(1)$) is consistent with the NTK parameterization used to show constancy of tangent kernel in~\cite{jacot2018neural,lee2019wide}.
\end{remark}


{\bf Hessian matrix.} 
We observe that 
the Hessian matrix $H$ of the neural network $f$ is sparse, specifically, diagonal:
$$
    H_{ij} =\partial^2 f/ \partial w_i \partial w_j = \frac{1}{\sqrt{m}}v_i\sigma''(w_ix)x^2\, \mathbbm{1}_{\{i=j\}}.
$$
Consequently, if the input $x$ is bounded, say $|x| \le C$, the spectral norm of the Hessian $H$ is
\begin{equation}\label{eq:scaling_hessian}
\|H\| = \max_{i\in [m]} |H_{ii}| = \frac{x^2}{\sqrt{m}}\max_{i\in [m]}|v_i\sigma''(w_ix)|\le \frac{1}{\sqrt{m}}\beta_{\sigma}C^2 = O\big(\frac{1}{\sqrt{m}}\big).
\end{equation}
In the limit of $m\to \infty$, the spectral norm $\|H\|$ converges to 0.

{\bf Tangent kernel and gradient.}
On the other hand, the magnitude of the norm of the tangent kernel of $f$ is of order $\Theta(1)$ in terms of $m$. 
Specifically,
for each diagonal entry we have
\begin{equation}\label{eq:normofgradient}
    K_{(x,x)}(\rvw) =  \|\nabla_{\rvw} f(\rvw;x)\|^2 = \frac{1}{m}\sum_{i=1}^m x^2(\sigma'(w_ix))^2 = \Theta(1).
\end{equation}
In the limit of $m\to \infty$, $ K_{(x,x)}(\rvw) = x^2\mathbb{E}_{w\sim \mathcal{N}(0,1)}[(\sigma'(wx))^2]$. Hence the trace of tangent kernel is also $\Theta(1)$.
Since the tangent kernel is a positive definite matrix of size independent of $m$, the norm is of the same order as the trace.

Therefore, from Eq.~(\ref{eq:scaling_hessian}) and Eq.~(\ref{eq:normofgradient}) we observe that the tangent kernel scales  as $\Theta(1)$ while the norm of the Hessian scales as $O(1/\sqrt{m})$ with the size of the neural network $f$. 
Furthermore, as $m\to\infty$, the norm of the Hessian converges to zero and, by Proposition~\ref{prop:Hessiantokernel}, the tangent kernel becomes constant. 


\paragraph{Why does the Hessian become small with increasing width?}
 So why should there be a discrepancy between the scaling of the Hessian spectral norm and the norm of the gradient? This is not a trivial question. There is no intrinsic reason why  second and first order derivatives should scale differently with the size of an arbitrary model. In the rest of this subsection we analyze the source of that phenomenon in wide neural networks, connecting it to  disparity of  different norms in high dimension.


Specifically, we show that the Hessian spectral norm is controlled by ${\infty}$-norm of the vector $\|{\partial f}/{\partial \alpha}\|_\infty$. In contrast, 
the tangent kernel and the norm of the gradient are controlled by its  Euclidean norm $\|{\partial f}/{\partial \alpha}\|$.
The disparity between these norms is the underlying reason for the transition to linearity in the limit of infinite width.



\noindent {\bf $\bullet$ Hessian is controlled by  $\|{\partial f}/{\partial \alpha}\|_{\infty}$.} Given a model $f$ in Eq.~(\ref{eq:defi1layer}), its Hessian matrix $H(f)$ is defined as  \begin{equation}
    H(f)=\partial^2 f/\partial \rvw^2 = \sum_{i=1}^m \frac{\partial f}{\partial \alpha_i}\frac{\partial^2 \alpha_i}{\partial \rvw^2},
\end{equation} 
where $\frac{\partial^2 \alpha_i}{\partial \rvw^2}$ are the components of the  order 3 tensor of partial derivatives $\frac{\partial^2\alpha}{\partial \rvw^{2}}$. When there is no ambiguity, we suppress the argument and denote the Hessian matrix by $H$.   
By Proposition~\ref{prop:221} (essentially the Holder's inequality: $|\rva^T\rvb| \leq \|\rva\|_1\|\rvb\|_\infty$), we have
\begin{equation}
    \|H\| \le 
    \left\|\frac{\partial^2\alpha}{\partial \rvw^{2}}\right\|_{2,2,1}\left\|\frac{\partial f}{\partial \alpha}\right\|_{\infty}.
\end{equation}
For this $1$-hidden layer network, the tensor $\frac{\partial^2\alpha}{\partial \rvw^{2}}$ is given by
\begin{equation}\label{eq:hessian_infinity_1_layer}
    \left(\frac{\partial^2\alpha}{\partial \rvw^{2}}\right)_{ijk} = \frac{\partial^2\alpha_k}{\partial w_i\partial w_j} = \sigma'(w_kx)x\cdot\mathbb{I}_{\{i=j=k\}},
\end{equation}
where $\mathbb{I}_{\{\}}$ is the indicator function.
By definition of the $(2,2,1)$-norm we have
\begin{equation}
    \left\|\frac{\partial^2\alpha}{\partial \rvw^{2}}\right\|_{2,2,1} = \sup_{\|\rvv_1\|=\|\rvv_2\|=1} \sum_{k=1}^m \sigma'(w_kx)x (\rvv_1)_k (\rvv_2)_k \le \mathsf{L}_{\sigma}x\sup_{\|\rvv_1\|=\|\rvv_2\|=1}\sum_{k=1}^m(\rvv_1)_k (\rvv_2)_k \le \mathsf{L}_{\sigma}x.\nonumber
\end{equation}
Thus, we conclude that the Hessian spectral norm $\|H\| = O\left(\left\|{\partial f}/{\partial \alpha}\right\|_{\infty}\right)$.

{
{$\bullet$ \bf Tangent kernel and the gradient are controlled by  $\|{\partial f}/{\partial \alpha}\|$.} Note that the norm of the tangent kernel is lower bounded by the average of diagonal entries: $\|K\| \ge \frac{1}{n}\sum_{i=1}^nK_{(x_i,x_i)}$, where $n$ is the size of the dataset.
Consider an arbitrary diagonal entry $K_{(x,x)}$ of the tangent kernel matrix. 
\begin{equation}
    K_{(x,x)} = \|\nabla_{\rvw}f(\rvw;x)\|^2 = \left\|\frac{\partial \alpha}{\partial \rvw}\frac{\partial f}{\partial \alpha}\right\|^2.
\end{equation}

Note that, $\frac{\partial \alpha}{\partial \rvw}$ is a diagonal matrix with  $\frac{\partial \alpha_i}{\partial w_i} = \sigma'(w_ix)x$. By the Lipschitz continuity of $\sigma(\cdot)$, $\left\|\frac{\partial \alpha}{\partial \rvw}\right\|$ is  finite. Therefore,  the tangent kernel is  of the same order as the $2$-norm $\left\|{\partial f}/{\partial \alpha}\right\|$.



{$\bullet$ \bf The discrepancy between the norms.} For the network in Eq.~(\ref{eq:defi1layer}) we have  $\frac{\partial f}{\partial \alpha} = \frac{1}{\sqrt{m}}\rvv$. Hence, 
\begin{equation}
    \left\|\frac{\partial f}{\partial \alpha}\right\|_{\infty} = \frac{1}{\sqrt{m}}, 
    ~~~ \left\|\frac{\partial f}{\partial \alpha}\right\| = 1.
\end{equation}

The transition to linearity stems from this observation and the fact discussed above that the Hessian norm scales as $\left\|\frac{\partial f}{\partial \alpha}\right\|_{\infty}$, while the tangent kernel is of the same order as $\left\|\frac{\partial f}{\partial \alpha}\right\|$.


In what follows, we show that this is a general principle applicable to wide neural networks. We start by analyzing  two hidden layer neural networks, which are mathematically similar to the general case, but much less complex in terms of the notation.

}

\subsection{Two hidden layer neural networks}\label{sec:2-hidden}
Now, we demonstrate that analogous results  hold for $2$-hidden layer neural networks. 
Consider the $2$-hidden layer neural network:
\begin{equation}\label{eq:new_nn}
    f(W_1,W_2;x) = \frac{1}{\sqrt{m}}\rvv^T\sigma\left(\frac{1}{\sqrt{m}}W_2 \sigma(W_1 \rvx)\right),~~ W_1\in \mathbb{R}^{m\times d}, ~ W_2\in \mathbb{R}^{m\times m}, \rvx \in \mathbb{R}^{d}.
\end{equation}
We denote the output of the first hidden layer by $\alpha^{(1)}(W_1;\rvx) = \sigma(W_1 \rvx)$ and the output of the second hidden layer by $\alpha^{(2)}(W_1,W_2;\rvx) = \sigma\left(\frac{1}{\sqrt{m}}W_2 \sigma(W_1 \rvx)\right)$.

{\bf Hessian is controlled by  $\|{\partial f}/{\partial \alpha}\|_{\infty}$.} Similarly to Eq.(\ref{eq:hessian_infinity_1_layer}), we can bound the Hessian spectral norm by ${\infty}$-norms of $\partial f/\partial \alpha^{(1)}$ and $\partial f/\partial \alpha^{(2)}$.




\begin{prop}
\begin{align}
    \| H\| 
    &\leq\left(\left\| \frac{\partial \alpha^{(1)}}{\partial W_1}\right\|^2\left\|\frac{\partial^2 \alpha^{(2)}}{(\partial \alpha^{(1)})^2}\right\|_{2,2,1}
    + 2\left\| \frac{\partial \alpha^{(1)}}{\partial W_1}\right\| \left\|\frac{\partial^2 \alpha^{(2)}}{\partial W_2 \partial \alpha^{(1)}}\right\|_{2,2,1}
    + \left\|\frac{\partial^2\alpha^{(2)}}{\partial W_2^2}\right\|_{2,2,1}\right)\left\|\frac{\partial f}{\partial \alpha^{(2)}} \right\|_\infty\nonumber\\
    &+\left\|\frac{\partial^2 \alpha^{(1)}}{\partial W_1^2}\right\|_{2,2,1}\left\|\frac{\partial f}{\partial \alpha^{(1)}}\right\|_\infty.\label{eq:hessian_infinity_2_layer}
\end{align}
Here $\frac{\partial}{\partial W_l}$  denotes partial derivatives w.r.t. each element of $W_l$, i.e. after flattening  the matrix $W_l$ .

\end{prop}
As this Proposition is  a special case of Theorem~\ref{thm:hessian_infinity}, we omit the proof.

When $W_1,W_2$ are initialized as random Gaussians,  every term in Eq.~(\ref{eq:hessian_infinity_2_layer}), except for $\left\|{\partial f}/{\partial \alpha^{(1)}}\right\|_\infty$ and $\left\|{\partial f}/{\partial \alpha^{(2)}}\right\|_\infty$, is of order $O(1)$, with high probability   within a ball of a finite radius (see the discussion in Subsection \ref{subsec:dnn_linear} for details).

 
 Hence, just like the one hidden layer case, the magnitude of Hessian spectral norm is controlled by these ${\infty}$-norms:
\begin{equation}
    \|H\| = O\left(\left\|\frac{\partial f}{\partial \alpha^{(1)}}\right\|_{\infty}+\left\|\frac{\partial f}{\partial \alpha^{(2)}}\right\|_{\infty}\right).
\end{equation}


{
{\bf Tangent kernel and the gradient are controlled by  $\|{\partial f}/{\partial \alpha}\|$.}
A diagonal entry of the kernel matrix can be decomposed into 
\begin{eqnarray*}
K_{(x,x)} = \|\nabla_{W_1}f(W_1,W_2;x)\|^2 + \|\nabla_{W_2}f(W_1,W_2;x)\|^2 = \left\|\frac{\partial\alpha^{(1)}}{\partial W_1}\frac{\partial f}{\partial \alpha^{(1)}}\right\|^2 +  \left\|\frac{\partial\alpha^{(2)}}{\partial W_2}\frac{\partial f}{\partial \alpha^{(2)}}\right\|^2,
\end{eqnarray*}
with each additive term being related to each layer. As the matrix $\partial \alpha^{(l)}/\partial W_l$ and the vector $\partial f/\partial\alpha^{(l)}$ are independent from each other and random at initialization, we expect $\left\|\frac{\partial\alpha^{(l)}}{\partial W_l}\frac{\partial f}{\partial \alpha^{(l)}}\right\|^2$ to be of the same order as $\left\|\frac{\partial\alpha^{(l)}}{\partial W_l}\right\|^2\left\|\frac{\partial f}{\partial \alpha^{(l)}}\right\|^2$, for $l=1,2$. 
}

\subsection{Multilayer neural networks}\label{subsec:dnn_linear}

Now, we extend the analysis to general deep neural networks. 

First, we show that, in parallel to  one and two hidden layer networks, the Hessian spectral norm and the tangent kernel of a multilayer neural network are controlled by ${\infty}$-norms and $2$-norms of the vectors $\partial f/\partial \alpha^{(l)}$, respectively. Then we show that the magnitudes of the two types of vector norms scales differently with respect to the network width. 

We consider a  general form of a deep neural network $f$ with a linear output layer:
\begin{align}\label{eq:generalnn}
 &\alpha^{(0)} = \rvx, \nonumber\\
 &\alpha^{(l)}= \phi_{l}(\rvw^{(l)};\alpha^{(l-1)}), \ \forall l = 1,2,\cdots, L,\nonumber\\
 &f = \frac{1}{\sqrt{m}}\rvv^T \alpha^{(L)},
\end{align}
where each vector-valued function $\phi_{l}(\rvw^{(l)};\cdot): \mathbb{R}^{m_{l-1}} \rightarrow \mathbb{R}^{m_l}$, with parameters $\rvw^{(l)} \in \mathbb{R}^{p_l}$, is considered as a layer of the network, and $m=m_L$ is the width of the last hidden layer.  This definition  includes the standard fully connected, convolutional (CNN) and residual (ResNet) neural networks as special  cases.

{
{\bf Initialization and parameterization.} In this paper, we consider the NTK initialization/ parameterization~\cite{jacot2018neural}, under which the constancy of the tangent kernel had been initially observed.  Specifically, the parameters, (weights), $\rmW := \{\rvw^{(1)},\rvw^{(2)},\cdots ,\rvw^{(L)},\rvw^{(L+1)}:=\rvv\}$ are drawn i.i.d. from a standard Gaussian, i.e., $w_i^{(l)}\sim \mathcal{N}(0,1)$, at initialization, denoted as $\rmW_0$. The factor $1/\sqrt{m}$ in the output layer is required by the NTK parameterization in order that the output $f$ is of order $\Theta(1)$. Different parameterizations (e.g., LeCun initialization: $w_i^{(l)}\sim \mathcal{N}(0,1/m)$)  rescale the tangent kernel and the Hessian by the same factor, and thus do not change our conclusions (see Appendix~\ref{secapp:parameterization}).

\subsubsection{Bounding the Hessian}
To simplify the notation, we start by defining the following useful quantities:
\begin{align}
\mathcal{Q}_\infty(f) \triangleq &\max_{1 \leq l \leq L}\left\{\left\|\frac{\partial{f}}{\partial\alpha^{(l)}}\right\|_\infty \right\},\\ ~~~~~ \mathcal{Q}_{L}(f) \triangleq   &\max_{1 \leq l\leq L}\left\{ \left\|\frac{\partial \alpha^{(l)}}{\partial \rvw^{(l)}}\right\|\right\},\nonumber
\\
    \mathcal{Q}_{2,2,1}(f) \triangleq &\max_{1\leq l_1<l_2<l_3\leq L}\left\{ \left\|\frac{\partial^2\alpha^{(l_1)}}{\partial\rvw^{(l_1)2}}\right\|_{2,2,1}, \left\|\frac{\partial \alpha^{(l_1)}}{\partial \rvw^{(l_1)}} \right\|\left\|\frac{\partial^2\alpha^{(l_2)}}{\partial \alpha^{(l_2-1)}\partial \rvw^{(l_2)}}\right\|_{2,2,1},\right.\nonumber\\
   &\quad\quad\quad\quad\quad\quad~\left.\left\|\frac{\partial\alpha^{(l_1)}}{\partial\rvw^{(l_1)}}\right\|\left\|\frac{\partial\alpha^{(l_2)}}{\partial\rvw^{(l_2)}}\right\|\left\|\frac{\partial^2\alpha^{(l_3)}}{(\partial\alpha^{(l_3-1)})^2}\right\|_{2,2,1}\right\}.\label{eq:defi_quantities}
\end{align}
\begin{remark}
It is important to note that the quantity $\mathcal{Q}_\infty(f)$ is simply the maximum of the $\infty$-norms $\left\|{\partial{f}}/{\partial\alpha^{(l)}}\right\|_\infty$, and that $\mathcal{Q}_{L}(f)$ and $ \mathcal{Q}_{2,2,1}(f)$ are independent of the vectors $\partial f/\partial \alpha^{(l)}, l\in[L]$.
\end{remark}

The Hessian spectral norm is bounded by these quantities via the following theorem (see Appendix~\ref{secapp:thm_general} for the proof).
\begin{thm}\label{thm:hessian_infinity}
 Consider a $L$-layer neural network in the form of Eq.(\ref{eq:generalnn}). 
 For any $\rmW$ in the parameter space, the following inequality holds:
 \begin{align*}
    \| {H}(f)\| \leq C_1 \mathcal{Q}_{2,2,1}(f)\mathcal{Q}_\infty(f)+ \frac{1}{\sqrt{m}}C_2 \mathcal{Q}_L(f),
 \end{align*}
where $C_1 = L(L^2\mathsf{L}_{\phi}^{2L}+L\mathsf{L}_{\phi}^L+1)$ and $C_2 = L\mathsf{L}_{\phi}^L$. 
\end{thm}

\begin{remark}
The factor $1/\sqrt{m}$ in the second term comes from the definition of the output layer in Eq.~(\ref{eq:generalnn}) and is useful to make sure the model output at initialization is of the same order as the ground truth labels.
\end{remark}

{\bf Tangent kernel and $2$-norms.}
A diagonal entry of the kernel matrix can be decomposed into 
\begin{eqnarray*}
K_{(x,x)} = \sum_{l=1}^L\|\nabla_{\rvw^{(l)}}f(\rmW;x)\|^2 = \sum_{l=1}^{L}\left\|\frac{\partial\alpha^{(l)}}{\partial \rvw^{(l)}}\frac{\partial f}{\partial \alpha^{(l)}}\right\|^2
\end{eqnarray*}
with each additive term being related to each layer. As before, we expect each term $\left\|\frac{\partial\alpha^{(l)}}{\partial \rvw^{(l)}}\frac{\partial f}{\partial \alpha^{(l)}}\right\|^2$ has the same order as $\left\|\frac{\partial\alpha^{(l)}}{\partial \rvw^{(l)}}\right\|^2\left\|\frac{\partial f}{\partial \alpha^{(l)}}\right\|^2$.

\subsubsection{Small Hessian spectral norm and constant tangent kernel}
In the following,  we apply Theorem~\ref{thm:hessian_infinity} to fully connected neural networks, and show that the corresponding Hessian spectral norm scales as $\tilde{O}(1/\sqrt{m})$, in a region with finite radius. 

To simplify our analysis, we make the following assumption.

{\bf Assumptions.} We assume the hidden layer width $m_{l} = m$ for all $l \in [L]$, the number of parameters in each layer $p_l \ge m$, and the output is a scalar.\footnote{The assumption $m_{l} = m$ is to simplify the analysis, as we discuss below we only need $m_{l} \ge m$.} 
We assume that  (vector-valued) layer functions $\phi_{l}(\rvw;\alpha), l\in [L],$ are $\mathsf{L}_{\phi}$-Lipschitz continuous and twice differentiable with respect to input $\alpha$ and  parameters $\rvw$.


A fully connected neural network has  the form as in Eq.(\ref{eq:generalnn}), with each layer function specified by
\begin{equation}\label{eq:fcnlayer}
    \alpha^{(l)} = \sigma(\tilde{\alpha}^{(l)}), \ \  \tilde{\alpha}^{(l)} =  \frac{1}{\sqrt{m}}W^{(l)}\alpha^{(l-1)}, \ \textrm{for }\ l\in [L],
\end{equation}
where $\sigma(\cdot)$ is a  $\mathsf{L}_{\sigma}$-Lipschitz continuous, $\beta_{\sigma}$-smooth activation function, such as $sigmoid$ and $tanh$. The layer parameters $W^{(l)}$ are reshaped into an $m\times m$ matrix. The Euclidean norm of $\rmW$ becomes: $\|\rmW\|= (\sum_{l=1}^L\|W^{(l)}\|_F^2)^{1/2}$.

With high probability over the Gaussian random initialization, we have the following lemma to bound the quantities $\mathcal{Q}_{\infty}(f)$, $\mathcal{Q}_{2,2,1}(f)$ and $\mathcal{Q}_L(f)$ in a neighborhood of $\rmW_0$:
\begin{lemma}\label{lemma:fcn_quantities}
  Consider a fully connected neural network $f(\rmW;\rvx)$ with linear output layer and Gaussian random initialization $\rmW_0$. Given any fixed $R>0$, at any point $\rmW \in B(\rmW_0,R):= \{\rmW: \|\rmW - \rmW_0\| \le R\}$,   with high probability over the initialization,  the quantity 
  \begin{equation}
    \mathcal{Q}_{\infty}(f) = \tilde{O}(1/\sqrt{m}), ~~~~\mathcal{Q}_{2,2,1}(f) = O(1), ~~~~ \mathcal{Q}_L(f) = O(1), ~~ \textrm{w.r.t. $m$.}
  \end{equation}
\end{lemma}
See the proof of the lemma in Appendix~\ref{secapp:fcn_quantities}.
Applying this lemma to Theorem~\ref{thm:hessian_infinity}, we immediately obtain the following theorem:
\begin{thm}\label{thm:fcn}
Consider a fully connected neural network $f(\rmW;\rvx)$ with linear output layer and Gaussian random initialization $\rmW_0$. Given any fixed $R>0$, and any  $\rmW \in B(\rmW_0,R):= \{\rmW: \|\rmW - \rmW_0\| \le R\}$,   with high probability over the initialization, the Hessian spectral norm satisfies the following:
\begin{equation}
    \|H(\rmW)\| = \tilde{O}\left(1/{\sqrt{m}}\right). 
\end{equation}
\end{thm}

{
\begin{remark}
We note that the above theorem  also applies to more general networks that have different hidden layer widths, as long as the width of each layer is larger than $m$. See Theorem~\ref{cor:different_width}below.
\end{remark}
}
In the limit of $m\to\infty$, the spectral norm of the Hessian $\|H(\rmW)\|$ converges to $0$, for all $\rmW \in B(\rmW_0,R)$.  By Proposition~\ref{prop:Hessiantokernel}, this immediately implies constancy of tangent kernel and linearity of the model, in the ball $B(\rmW_0,R)$. 

On the other hand, the tangent kernel is of order $\Theta(1)$ (see for example \cite{du2018gradientdeep}, where the smallest eigenvalue of the tangent kernel is lower bounded by a width-independent constant). Intuitively, the order of tangent kernel stems from the fact that the $2$-norms $\|\partial f/\partial \alpha^{(l)}\|$ are of order $\Theta(1)$.

\begin{remark} By the optimization theory built in our work~\cite{liu2020toward}, a finite radius $R$ is enough to include the gradient descent solution, for the square loss. Hence, for very wide networks, the tangent kernel is constant during gradient descent training.
\end{remark}
}
\subsubsection{Neural networks with hidden layers of different width and general architectures}
Our analysis above is applicable to other common neural architectures including Convolutional Neural Networks (CNN) and ResNets, as well as networks with a mixed architectural types. 
Below we briefly highlight the main differences from the fully connected case.  Precise statements can be found in~ Appendix~\ref{secapp:cnn_resnet}.

{\bf CNN.} A convolutional layer maps a hidden layer ``image" $\alpha^{(l)}\in\mathbb{R}^{p\times q\times m_l}$ to the next layer $\alpha^{(l+1)}\in \mathbb{R}^{p\times q\times m_{l+1}}$, where $p$ and $q$ are the sizes of images in the spatial dimensions and $m_l$ and $m_{l+1}$ are the number of channels. Note that the number of channels, which can be arbitrarily large, defines the width of CNN, while spatial dimensions, $p$ and $q$, are always fixed.

The key observation is that a convolutional layer is ``fully connected'' in the channel dimension. In contrast, the convolutional operation, which is sparse, is only within the spatial dimensions. Hence, we can apply our analysis to the channel dimension with only minor modifications. As the spatial dimension sizes are independent of the network width, the convolutional operation only contributes constant factors to our analysis. Therefore, our norm analysis extends to the CNN setting.

{\bf ResNet.} A residual layer has the same form as Eq.(\ref{eq:fcnlayer}), except that the activation $\alpha^{(l)} = \sigma(\tilde{\alpha}^{(l)})+\alpha^{(l-1)}$, which results in an additional identity matrix $I$ in the first order derivative w.r.t. $\alpha^{(l-1)}$. As shown in Appendix~\ref{secapp:cnn_resnet}, the appearance of $I$ does not affect the orders of both the $\infty$-norms and $2$-norms of $\partial f/\partial \alpha^{(l)}$, as well as the related $(2,2,1)$-norms. Hence, the analysis above applies.


{\bf Architecture with mixed layer types.} Neural networks used in practice are often a mixture of different layer types, e.g., a series of convolutional layers followed by fully connected layers. Since our analysis relies on layer-wise quantities, our results extend to such networks.

We have the following general theorem which summarizes our theoretical results.
\begin{thm}[Hessian norm is controlled by the minimum hidden layer width] \label{cor:different_width}
Consider a general neural network $f(\rmW;x)$ of the form Eq.(\ref{eq:generalnn}), which can be a fully connected network, CNN, ResNet or a mixture of these types. Let $m$ be the minimum of the hidden layer widths, i.e., $m= \min_{l\in[L]} m_l$. Given any fixed $R>0$, and any  $\rmW \in B(\rmW_0,R):= \{\rmW: \|\rmW - \rmW_0\| \le R\}$,   with high probability over the initialization, the Hessian spectral norm satisfies the following:
\begin{equation}
    \|H(\rmW)\| = \tilde{O}\left(1/{\sqrt{m}}\right). 
\end{equation}
\end{thm}

\section{Constant tangent kernel is not a general property of wide networks}\label{sec:non_constant}
In this section, we show that  a class of infinitely wide neural networks with  {\it non-linear} output,  do not generally have constant tangent kernels. It also demonstrates that a linear output layer is a necessary condition for transition to linearity.


We consider the neural network  $\tilde{f}$: 
\begin{equation}
\tilde{f}(\rvw;\rvx) := {\phi}(f(\rvw;\rvx)).\label{eq:fhat}
\end{equation}
where $f(\rvw;\rvx)$ is a sufficiently wide neural network with linear output layer  considered in Section~\ref{sec:linear_ouput_nn}, and  
$\phi(\cdot)$ is a non-linear twice-differentiable activation function. 
The only  difference between $f$ and $\tilde{f}$  
is that $\tilde{f}$ has a non-linear output layer. As we shall see, this difference leads to a non-constant tangent kernel during training, as well as a different scaling behavior of the Hessian spectral norm.



{\bf Tangent kernel of $\tilde{f}$.}
The gradient of  $\tilde{f}$ is given by
$
    \nabla_{\rvw} \tilde{f}(\rvw; \rvx) = {\phi}'(f(\rvw;\rvx))\nabla_{\rvw} f(\rvw;\rvx).
$
Hence, each diagonal entry of the tangent kernel of $\tilde{f}$ is
\begin{equation}
    \tilde{K}_{(\rvx,\rvx)}(\rvw) = \|\nabla_{\rvw} \tilde{f}(\rvw; \rvx)\|^2 = {\phi}^{\prime 2}(f(\rvw;\rvx)) K_{(\rvx,\rvx)}(\rvw),
\end{equation}
where $K_{(\cdot,\cdot)}(\rvw)$ is the tangent kernel of $f$. By Eq.(\ref{eq:normofgradient}) we have 
$
\tilde{K}_{(\rvx,\rvx)}(\rvw) = \Theta(1),
$ which is of the same order as $K_{(\rvx,\rvx)}(\rvw)$. 

Yet, unlike $K_{(\cdot,\cdot)}(\rvw)$, the kernel $\tilde{K}_{(\cdot,\cdot)}(\rvw)$ changes significantly during training, even as $m \to \infty$ (with a change of the order of $\Theta(1)$).  
To prove that, it is enough to verify that at least one entry of $\tilde{K}_{(\rvx,\rvx)}(\rvw)$ has a change of $\Theta(1)$, for an arbitrary $\rvx$. 
Consider a diagonal entry. For any $\rvw$, we have
\begin{align*}
    &\big|\tilde{K}_{(\rvx,\rvx)}(\rvw)-\tilde{K}_{(\rvx,\rvx)}(\rvw_0)\big| = \left|{\phi}^{\prime 2}(f(\rvw;\rvx)){K}_{(\rvx,\rvx)}(\rvw) - {\phi}^{\prime 2}(f(\rvw_0;\rvx)){K}_{(\rvx,\rvx)}(\rvw_0)\right| \\
    &\ge \underbrace{\left|{\phi}^{\prime 2}(f(\rvw;\rvx))-{\phi}^{\prime 2}(f(\rvw_0;\rvx))\right|\cdot{K}_{(\rvx,\rvx)}(\rvw_0)}_A  -\underbrace{{\phi}^{\prime 2}(f(\rvw;\rvx))\cdot \left| {K}_{(\rvx,\rvx)}(\rvw)-{K}_{(\rvx,\rvx)}(\rvw_0) \right|}_B.
 \end{align*}

We note that the term $B$ vanishes as $m\to \infty$ due to the constancy of the tangent kernel of $f$. However the term $A$ is generally of the order $\Theta(1)$, when $\phi$ is non-linear\footnote{If $\phi$ is linear, the term $A$ is identically zero.}.  To see that consider any solution $\rvw^*$ such that $f(\rvw^*;\rvx)=y$ (which exists for over-parameterized networks). Since $f(\rvw_0;\rvx)$ is generally not equal to $y$, we obtain the result.

{\bf "Lazy training" does not explain constancy of NTK.}
From the above analysis, we can see that, even with the same parameter settings as $f$ (i.e., same initial parameters and same parameter change), network $\tilde{f}$ does not have constant tangent kernel, while the tangent kernel of $f$ is constant. This implies that the constancy of the tangent kernel cannot be explained in terms of the magnitude of the parameter change from initialization.
Instead, it depends on the structural properties of the network, such as the linearity of the output layer. Indeed, as we discuss next, when the output layer is non-linear, the Hessian norm of $\tilde{f}$ no longer decreases with the width of the network.



{\bf Hessian matrix of $\tilde{f}$.} 
 The Hessian matrix of $\tilde{f}$ is
 \begin{equation}
     \tilde{H} := \frac{\partial^2 \tilde{f}}{\partial \rvw^2} = \phi''(f) \nabla_{\rvw}f(\nabla_{\rvw}f)^T + \phi'(f)H,
 \end{equation}
where $H$ is the Hessian matrix of model $f$. Hence, the spectral norm satisfies
\begin{equation}\label{eq:tilde_H_bound}
    \|\tilde{H}\| \ge |\phi''(f)|\cdot \left\|\nabla_{\rvw}f\right\|^2 - |\phi'(f)|\cdot\|H\|.
\end{equation}
Since, as we already know, $\lim_{m\to\infty}\|H\|=0$, the second term $|\phi'(f)|\cdot\|H\|$ vanishes in the infinite width limit. However, the first term is always of order $\Theta(1)$, as long as $\phi$ is not linear. Hence, 
$\|\tilde{H}\| = \Omega(1)$, compared to $\|H\|=\tilde{O}({1}/{\sqrt{m}})$ for networks in Section~\ref{sec:linear_ouput_nn} and does not vanish as $m \to \infty$. 
{\begin{remark}
Note that the first term $|\phi''(f)|\cdot \left\|\nabla_{\rvw}f\right\|^2$ in Eq.(\ref{eq:tilde_H_bound}) has the same order as the square {\it $2$-norm} $\|\partial f/\partial \alpha^{(l)}\|^2$, instead of $\infty$-norm which controls the second term. 
Therefore,  the spectral norm of $\tilde{H}$ is no longer of order $O(1/\sqrt{m})$, in contrast to that of $H$.
\end{remark}
}

{\bf Neural networks with bottleneck.} Here, we show another type of neural networks, bottleneck networks, that does not have constant tangent kernel, by breaking the $\tilde{O}({1}/{\sqrt{m}})$-scaling  of Hessian spectral norm. 
Consider a neural network with fully connected layers. Here, we assume all the hidden layers are arbitrarily wide, except one layer, $l \ne L$, has a narrow width. For example, let the bottleneck width $m_b=1$. Now, the $(l+1)$-th fully connected layer, Eq.(\ref{eq:fcnlayer}), reduces to
$$
    \alpha^{(l+1)} = \sigma(\rvw^{(l+1)}\alpha^{(l)}),
$$
with $\alpha^{(l)} \in \mathbb{R}$ and $\rvw^{(l+1)} \in \mathbb{R}^m$. 
In this case, the $(2,2,1)$-norm of the order $3$ tensor $\frac{\partial^2 \alpha^{(l+1)}}{(\partial \alpha^{(l)})^2}\in \mathbb{R}^{1\times 1 \times m}$ is
\begin{equation}
   \Big\|\frac{\partial^2 \alpha^{(l+1)}}{(\partial \alpha^{(l)})^2}\Big\|_{2,2,1} = \sum_{i=1}^m |(w^{(l+1)}_i)^2\sigma''(w^{(l+1)}_i\alpha^{(l)})| = {\Theta}(m).
\end{equation}
This makes the quantity $\mathcal{Q}_{2,2,1}(f)$ to be the order of $O(m)$. Then, Theorem~\ref{thm:hessian_infinity} indicates that the Hessian spectral norm is no longer arbitrarily small, suggesting a non-constant tangent kernel during training. 

{
Indeed, as we prove below,
 the Hessian spectral norm is lower bounded by a positive constant, which in turn implies that the linearity does not hold for this kind of neural networks.

Specifically, consider a  bottleneck network with of the following form:
\begin{equation}\label{eq:bottleneck_example}
    f(\rmW;x) = \frac{1}{\sqrt{m}}\sum_{j=1}^m \rvw^{(4)} \sigma\left(\rvw_j ^{(3)} \frac{1}{\sqrt{m}}\sum_{i=1}^m\rvw^{(2)}_i \sigma(\rvw^{(1)}_i x)\right).
\end{equation}
This network has three hidden layers, where the first and third hidden layer have an arbitrarily large width $m$,
and the second hidden layer, as the bottleneck layer, has a width $m_b=1$. Each individual parameter is initialized by the standard norm distribution. 
For simplicity of the analysis, the activation function is  identity for the bottleneck layer is identity, and is quadratic for the first and third layers, i.e. $\sigma(z) = \frac{1}{2} z^2$. 

The following theorem gives a lower bound for the Hessian spectral norm $\|H\|$ in a ball around the initialization $\rmW_0$. 
\begin{thm}\label{thm:lowerbound_H}
Consider the bottleneck network $f(\rmW; x)$ defined in Eq. (\ref{eq:bottleneck_example}).  Given an arbitrary radius $R>0$, for any $\rmW\in B(\rmW_0, R)$ and any $\delta\in (0,1)$, the Hessian matrix $H(\rmW)$ of the model  satisfies 
\begin{equation}
    \|H(\rmW)\| \geq \frac{x^2}{4}\left(\frac{1}{2} -\frac{R^2}{m}\right)\left(\sqrt{3}C_1\delta^3  -\frac{3\log(4/\delta) R^4 + R^3}{\sqrt{m}}\right),
\end{equation}
for some constant $C_1>0$, with probability at least $1-2\delta- e^{-m/16}$.

In particularly, in the limit of $m\to\infty$, 
\begin{equation}
    \|H(\rmW)\| \ge \frac{\sqrt{3}}{8}C_1\delta^3x^2,
\end{equation}
with probability at least $1-2\delta$.
\end{thm}
See the proof in Appendix~\ref{secapp:proof_lowerbound_H}. With this lower bounded Hessian, Proposition~\ref{prop:linearity} directly implies that the linearity of the model does not hold for this network.
As Eq.~(\ref{eq:norm_lower_bound}) shows that $\|\rmW^* - \rmW_0\| = \Omega(1)$, our analysis implies the model is not linear, hence tangent kernel is not constant, along the optimization path.
In Section~\ref{subsec:numeric}, we empirically verify this finding.


}

In table \ref{tab:comparenn}, we summarize the key findings of this section and compare them with the case of neural networks with  linear output layer. 
\begin{table}[h!]
  \centering
  \begin{tabular}{|c|c|c|c|}
 \hline
  Network    &  Hessian norm & NTK  & Trans. to linearity (constant NTK)? \rule{0pt}{4ex} \rule[-2.4ex]{0pt}{0pt}\\
\hline
  linear output layer & $\tilde{O}(1/\sqrt{m})$ & $\Theta(1)$  &  {\bf Yes}  \rule{0pt}{3ex} \rule[-1.6ex]{0pt}{0pt}\\
\hline
   nonlinear output layer  & $\tilde{O}(1)$ & $\Theta(1)$  & {\bf No} \rule{0pt}{3ex} \rule[-1.6ex]{0pt}{0pt}\\
  \hline
  bottleneck  & $\tilde{O}(1)$ & $\Theta(1)$ &{\bf No} \rule{0pt}{3ex} \rule[-1.6ex]{0pt}{0pt}\\
  \hline
  \end{tabular}
  \vspace{10pt}
    \caption{Scaling of Hessian spectral norms of the models: linear output layer, non-linear output layer and bottleneck. Note: transition to linearity $=$ constant tangent kernel, in the infinite width limit. }
    \label{tab:comparenn}
    \vspace{-20pt}
  \end{table}

\section{Optimization of wide neural networks}\label{sec:optimization}

A number of recent analyses show convergence of gradient descent for wide neural networks~\cite{du2018gradientshallow,du2018gradientdeep,allen2019convergence,zou2018stochastic,arora2019fine,ji2019polylogarithmic,bartlett2019gradient}.  
While an extended discussion of optimization is  beyond the scope of this work,  we refer the interested reader to  our separate paper \cite{liu2020toward}.
The goal of this section is to clarify the  important difference between the (near-)linearity of large models and convergence of optimization by gradient descent. It is easy to see  that a wide model undergoing the transition to linearity   can be optimized by gradient descent if its tangent kernel is well-conditioned at the initialization point. The dynamics of such a model will be essentially the same as for a linear model, an observation originally made in~\cite{jacot2018neural}.

However near-linearity or, equivalently,  near-constancy of the tangent kernel is not necessary for successful optimization. What is needed is that the tangent kernel is well-conditioned along the optimization path, a far weaker condition. 

For a specific example, consider the non-linear output layer neural network $\tilde{f} = \phi(f)$, as defined in Eq.~(\ref{eq:fhat}). As is shown in Section~\ref{sec:non_constant}, this network does not have constant tangent kernel, even when the network width is arbitrarily large. The following theorem states that fast convergence of gradient descent still holds (also see Section~\ref{subsec:numeric} for empirical verification).
 \begin{thm}
Suppose the non-linear function $\phi(\cdot)$ satisfies $|\phi'(z)| \ge \rho > 0, \forall z\in \mathbb{R}$, and the network width $m$ is sufficiently large. Then, with high probability of the random initialization, there exists constant $\mu>0$, such that the gradient descent, with a small enough step size $\eta$, converges to a global minimizer of the square loss function $\L(\rvw)= \frac{1}{2}\sum_{i=1}^n (\tilde{f}(\rvw;\rvx_i)-y_i)^2$  with an exponential convergence rate:
 \begin{equation}
   \L(\rvw_t) \le (1-\eta\mu\rho^2)^t \L(\rvw_0).
 \end{equation}

\end{thm}
The analysis is based on the following reasoning. Convergence  of gradient descent methods relies on the condition number of the tangent kernel (see~\cite{liu2020toward}). It is not difficult to see that if the original model $f$ has a well conditioned tangent kernel, then the same holds for $\tilde{f}=\phi(f)$ as long as the the derivative of the activation function $\phi'$ is separated from zero. Since the tangent kernel of $f$ is not degenerate, the conclusion follows.

The technical result is a consequence of Corollary 8.1 in~\cite{liu2020toward}. 



\section{Numerical Verification}\label{subsec:numeric}

\begin{figure}
    \centering
    \begin{subfigure}[t]{0.48 \textwidth}
        \centering
        \includegraphics[width=\linewidth]{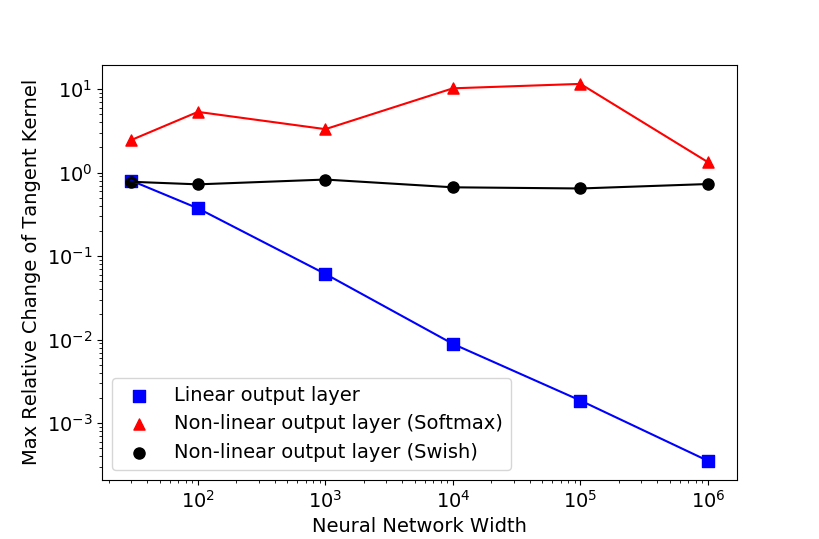} 
       \label{fig:epoch}
    \end{subfigure}
            \begin{subfigure}[t]{0.48\textwidth}
        \centering
        \includegraphics[width=\linewidth]{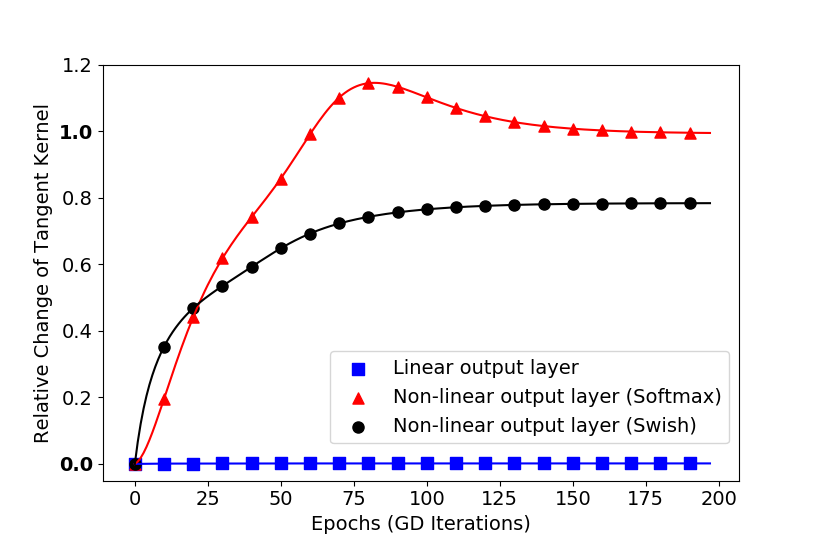} 
       \label{fig:width}
       \end{subfigure}
    \hfill
     \caption{Neural networks with {\it non-linear} output layer vs.  with {\it linear} output layer. Left panel (log scale): max change of tangent kernel $\Delta K$ from initialization to convergence w.r.t. the width $m$. Right panel: Evolution of tangent kernel change $\Delta K(t)$ as a function of epoch, width $m=10^4$. }
     \label{fig:epoch&width}
     \vspace{-5pt}
   \end{figure}
%

We conduct experiments to verify the non-constancy of tangent kernels for certain types of wide neural networks, as theoretically observed in Section~\ref{sec:non_constant}. 

Specifically, we use gradient descent to train each neural network described below  on a synthetic data  until convergence. We compute the following quantity to measure the max (relative) change of tangent kernel from initialization to convergence:
$\Delta K := \sup_{t>0}\|K(\rvw_t)-K(\rvw_0)\|_F/\|K(\rvw_0)\|_F.$ For a network that has a nearly constant tangent kernel during training, $\Delta K$ is expected to be close to $0$, while a network with a non-constant tangent kernel, $\Delta K$ should be $\Omega(1)$. Detailed experimental setup and data description are given in Appendix~\ref{secapp:experiment}.

{\bf Wide neural networks with non-linear output layers.}  We consider a shallow (i.e., with one hidden layer) neural network $\tilde{f}$ of the type in Eq.(\ref{eq:fhat}) that has a softmax layer or swish~\cite{ramachandran2017searching} activation on the output. As a comparison, we consider a neural network $f$ that has the same structure as $\tilde{f}$, except that the output layer is  linear.
\begin{wrapfigure}{R}{0.45\textwidth}
   \vspace{-10pt}
\begin{center}
  \includegraphics[width=0.43\textwidth]{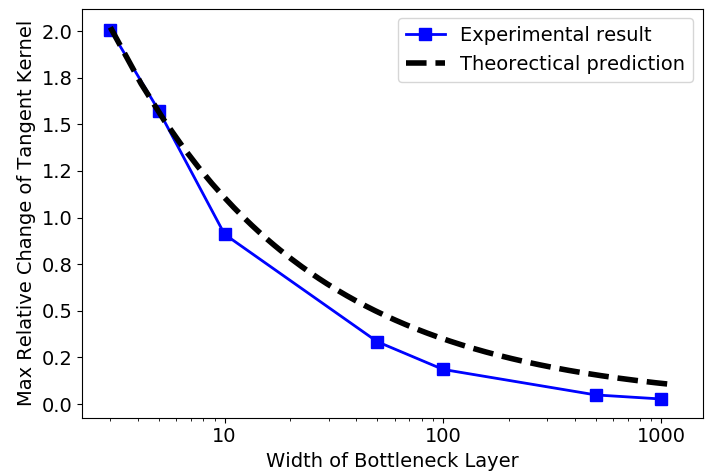} 
\end{center}
\caption{\label{fig:bottleneck}Networks with bottleneck. { Experimental results and theoretical prediction of }relative change of  tangent kernel from initialization to convergence, as a function of the bottleneck width.}
\vspace{-12pt} 
\end{wrapfigure}

We report the change of tangent kernels $\Delta K$ of $\tilde{f}$ and $f$, at different network width $m=\{30$, $10^2$, $10^3$, $10^4$, $10^5$, $10^6\}$.
The results are plotted in the left panel of Figure~\ref{fig:epoch&width}. We observe that, as the network width increases, the tangent kernel of $f$, which has a linear output layer, tends to be constant during training. However, the tangent kernel of $\tilde{f}$ which has a non-linear (softmax or swish) output layer, always takes significant change, even if the network width is large.

In Figure~\ref{fig:epoch&width}, right panel, we demonstrate the evolution of tangent kernel  with respect to the training time for a very wide neural network (width $m=10^4$).
 We see that, for the neural network with  a non-linear output layer, tangent kernel changes significantly from initialization, while tangent kernel of the linear output network is nearly unchanged during training.

\bf Wide neural networks with a bottleneck. \normalfont
We consider a fully connected neural network with $3$ hidden layers and a linear output layer. The second hidden layer, i.e., the bottleneck layer, has a width $m_b$ which is typically small, while the width $m$ of the other hidden layers are typically very large, $m=10^4$ in our experiment. For different bottleneck width $m_b = \{3$, $5$, $10$, $50$, $100$, $500$,$1000\}$, we train the network on a synthetic dataset using gradient descent until convergence, and compute $\Delta K$.

The change of tangent kernels for different bottleneck width is shown in Figure~\ref{fig:bottleneck}. We can see that a narrow bottleneck layer in a wide neural network prevent the neural tangent kernel from being constant during training. As expected, increasing the width of the bottleneck layer, makes the change of the tangent kernel smaller. We observe that { the scaling of the tangent kernel change with width follows close to $\Theta\left(1/\sqrt{m}\right)$ (dashed line in Figure~\ref{fig:bottleneck}) in alignment with our theoretical results (Theorem \ref{cor:different_width})}.

\section*{Acknowledgements}
The authors acknowledge support from NSF, the Simons Foundation and a Google Faculty Research Award. We thank James Lucas for 
 correcting the proof of Prop.~\ref{prop:Hessiantokernel}. The GPU used for the experiments was donated by Nvidia.
\printbibliography

 \newpage
 \appendix

 \section{Other Parameterization Strategies}\label{secapp:parameterization}
 Throughout the paper, our analysis is based on the NTK prameterization~\cite{jacot2018neural}, under which the constancy of tangent kernel is originally observed. 
 In this section,  we show that different parameterization strategies (e.g., LeCun initialization~\cite{lecun2012efficient} 
 : $w_{0;i}^{(l)}\sim \mathcal{N}(0,1/m)$) 
 do not change our conclusions. 
 
 Specifically, we show that, compared to the NTK prameterization, a different parameterization strategy only rescales the tangent kernel $K$ and the spectral norm of the Hessian $\|H\|$ by the same factor, hence the ratio between tangent kernel $K$ and Hessian spectral norm keeps the same and $\|H\| = o(\|K\|)$  still holds. This still implies that the tangent kernel is almost constant during training.
 
 Recall that we initialize the parameters $\rmW= \{\rvw^{(1)},\rvw^{(2)},\cdots ,\rvw^{(L)},\rvw^{(L+1)}:=\rvv\}$ of the general form of a deep neural network $f$, Eq.(\ref{eq:generalnn}) by a standard Gaussian, i.e. $w_i^{(l)}\sim \mathcal{N}(0,1)$. If we apply another parameterization strategy $\bar{\rmW}$ here, for example,  $\bar{w}_i^{(l)}\sim \mathcal{N}(0,\sigma_m^2)$, where $\sigma_m$ can be a function of $m$, we can see every $\bar{w}_i^{(l)} = \sigma_m w_i^{(l)}$ where $w_i^{(l)} \sim \mathcal{N}(0,1)$.
 
 Hence each layer function becomes
 \begin{align}
 &\alpha^{(l)}= \phi_{l}\left(\frac{1}{\sigma_m}\bar{\rvw}^{(l)};\alpha^{(l-1)}\right).
\end{align}
In this case, the gradient of the model $f$ w.r.t. the weights of layer $l$ is 
\begin{align}
    \frac{\partial f}{\partial \bar{\rvw}^{(l)}} =\frac{\partial \rvw^{(l)}} {\partial \bar{\rvw}^{(l)}} \frac{\partial f}{\partial \rvw^{(l)}}  = \frac{1}{\sigma_m}\frac{\partial f}{\partial \rvw^{(l)}}. 
\end{align}
And by the same reason, the Hessian of the model f w.r.t. the weights of layer $l_1$ and $l_2$ is 
\begin{align}
     \frac{\partial^2 f}{\partial \bar{\rvw}^{(l_1)}\partial \bar{\rvw}^{(l_2)}} = \frac{1}{\sigma_m^2}\frac{\partial^2 f}{\partial \rvw^{(l_1)}\partial \rvw^{(l_2)}}.
\end{align}

Therefore, it's easy to see the ratio of the norm of the tangent kernel to the norm of the Hessian keeps the same:
\begin{align}
    \frac{\|{K}(\bar{\rmW})\|}{\| H(\bar{\rmW})\|} =  \frac{\|\frac{1}{\sigma_m^2}K(\rmW)\|}{\|\frac{1}{\sigma_m^2}H(\rmW)\|} = \frac{\|K(\rmW)\|}{\|H(\rmW)\|}.
\end{align}
\paragraph{Example:  LeCun initialization/parameterization.}
In many practical machine learning tasks, it is popular to use the LeCun initialization/parameterization: each individual parameter $(W^{(l)}_0)_{ij} \sim \mathcal{N}(0,\frac{1}{m})$, while there is no factor $1/\sqrt{m}$ in the definition of the layer function, e.g., for fully connected layers
\begin{equation}
    \alpha^{(l+1)} = \sigma(W^{(l)}\alpha^{(l)}).
    \end{equation}
In this setting, the factor $\sigma_{m} = 1/\sqrt{m}$. Then, by the analysis above, we see that
\begin{equation}
    \|K\| = O(m), \ \ \|H\| = O(\sqrt{m}) = o(\|K\|).
\end{equation}
It is also interesting to note that, the Euclidean norm of the parameter change $\rvw^*-\rvw_0$ also scales:
\begin{equation}
    \|\rvw^*-\rvw_0\| = \Theta(1/\sqrt{m}).
\end{equation}

 \section{Experimental Setup}\label{secapp:experiment}
\paragraph{Dataset.} We use a synthetic dataset of size $N=60$ which contains $C=3$ classes. Each data point $(x,y)$ is sampled as follows: label $y$ is randomly sampled from $\{0,1,2\}$ with equal probability; given $y$, $x$ is drawn from the following distribution:
\begin{equation}
x\sim\left\{
\begin{aligned}
&\mathcal{N}(0,1), \quad &\mathrm{if} \; y=0; \\
&\mathcal{N}(10,1), \quad &\mathrm{if} \; y=1; \\
&\mathcal{N}(-10,1), \quad &\mathrm{if} \; y=2.
\end{aligned}
\right.
\end{equation}
We encode each $y_i\in \{0,1,2\}$ in $\{(x_i,y_i)\}_{i=1}^N$ by a one-hot vector $\rvy_i\in \{0,1\}^3$. And $\rvy_{i,j}$ means the $j$-th component of $\rvy_i$. We use this dataset for all the optimization tasks mentioned below.

\subsection{Wide neural networks with non-linear output layers}
\paragraph{Neural Networks.} In the experiments, we train three different neural networks:
\begin{itemize}
    \item Neural network with a linear output layer 
    \begin{equation}\label{eqapp:linearnn}
    f(\rvw,V,\rvb;x) = \frac{1}{\sqrt{m}}V  \sigma( \rvw x+\rvb),
    \end{equation}
     where $\rvw\in\mathbb{R}^{m}$ and $\rvb\in\mathbb{R}^{m}$ are weights and biases for the first layer and  $V \in\mathbb{R}^{3\times m}$ are the weights for the output layer, and $\sigma(\cdot)$ is the ReLU activation function. 

\item Neural network with a softmax-activated (non-linear) output layer
\begin{equation}\label{eqapp:softmaxnn}
    \tilde{f}_1(\rvw,V,\rvb;x) = \mathrm{Softmax}(f(\rvw,V,\rvb;x));
\end{equation}
\item Neural network with a swish-activated (non-linear) output layer
\begin{equation}\label{eqapp:swishnn}
    \tilde{f}_2(\rvw,V,\rvb;x) = \mathrm{Swish}(f(\rvw,V,\rvb;x)).
\end{equation}
Here the swish activation function is defined as $\mathrm{Swish}(\rvz) = \rvz\odot (1 + \exp(-0.1\cdot \rvz))^{-1}$, where $\odot$ is the element-wise multiplication.
\end{itemize}



\textbf{Optimization Tasks.} 
We combine the training of networks $f$ and $\tilde{f}_1$ together, by optimizing the following loss function:
\begin{equation}
    \mathcal{L}_1(\rvw,\rvv,\rvb) = -\frac{1}{N}\sum_{i=1}^{N}\sum_{j=1}^{C}\rvy_{i,j}\cdot log((\tilde{f}_1(\rvx_i)_j).
\end{equation}
In this combined training, networks $f$ and $\tilde{f}_1$ always have the same parameters during training, and the difference between $f$ and $\tilde{f}_1$ is the non-linearity on the output.


For the swish-activated network $\tilde{f}_2$, we 
minimize the square loss function:
\begin{equation}
\mathcal{L}_2(\rvw,\rvv,\rvb) = \frac{1}{N}\sum_{i=1}^{N}\| \tilde{f}_2(\rvx_i) - \rvy_{i}\|^2.
\end{equation}
We use gradient descent to minimize the loss functions until convergence is achieved (i.e. loss less than $10^{-4}$).  To measure the change of tangent kernels, we compute the max (relative) change of tangent kernel from initialization to convergence:
$\Delta K := \sup_{t>0}\|K(\rvw_t)-K(\rvw_0)\|_F/\|K(\rvw_0)\|_F.$ For each training, we take $10$ independent runs and report the average $\Delta K$.

We compare the tangent kernel changes $\Delta K$ of $f, \tilde{f}_1$ and $\tilde{f}_2$, at a variety of network widths, $m= 30$, $10^2$, $10^3$, $10^4$, $10^5$, $10^6$. 



\subsection{Wide neural networks with a bottleneck}
\paragraph{The Neural Network.} In the experiment, we use a fully connected neural network with $3$ hidden layers and a linear output layer. Its second hidden layer, i.e., the bottleneck layer has a width $m_b$, while the other hidden layers has a width $m$. Specifically, it is defined as:
 \begin{equation}
     f(\rmW;x) =  \frac{1}{\sqrt{m}}W_4\sigma\left(W_3\frac{1}{\sqrt{m}}W_2\sigma(W_1 x)\right),\label{eq:narrownn}
 \end{equation}
 where $W_1\in \mathbb{R}^{m\times 1}$, $W_2\in \mathbb{R}^{m_b\times m}$, $W_3\in \mathbb{R}^{m\times m_b}$, $W_4\in \mathbb{R}^{C\times m}$. Here we use ReLU as activation functions.
 \paragraph{Optimization Tasks.} We minimize the cross entropy loss:
 \begin{equation}
    \mathcal{L}(\rmW) = -\frac{1}{N}\sum_{i=1}^{N}\sum_{j=1}^{C}\rvy_{i,j}\cdot log((\tilde{f}(\rvx_i)_j),
\end{equation}
 where we denote $\mathrm{Softmax}(f)$ by $\tilde{f}$. Here, we let the network width $m = 10^4$, and investigate on different bottleneck width $m_b \in  \{3$, $5$, $10$, $50$, $100$, $500$,$1000\}$.
 
 For each bottleneck width, we use gradient descent to minimize the loss functions until convergence is achieved (i.e. loss less than $10^{-4}$) and compute the max (relative) change of tangent kernel from initialization to convergence:
$\Delta K := \sup_{t>0}\|K(\rvw_t)-K(\rvw_0)\|_F/\|K(\rvw_0)\|_F.$ For each training, take $10$ independent runs and report the average $\Delta K$.

 \section{Proof for Proposition~\ref{prop:linearity}}\label{secapp:linearity}
 \begin{proof}
 Recall that the tangent kernel is defined as
 \begin{equation}
    K_{ij}(\rvw) = \nabla f(\rvw;\rvx_i)^T\nabla f(\rvw;\rvx_j), \quad \textrm{for any inputs }\rvx_i, \rvx_j \in \mathbb{R}^{d}.
\end{equation}

\paragraph{Linearity of $f$ in $\rvw$  $\Rightarrow$ constancy of tangent kernel.} 

Since $f$ is linear in $\rvw$, $\nabla_{\rvw} f(\rvw;\rvx)$ is a constant vector in $\mathbb{R}^p$, for any given input $\rvx$. By the definition of the tangent kernel, each element $K_{ij}(\rvw)$ is constant, for any inputs $\rvx_i,\rvx_j$.

\paragraph{Constancy of tangent kernel $\Rightarrow$  linearity of $f$ in $\rvw$.}

It suffices to prove for every input $\rvx_i$,  function $f(\rvw;\rvx_i):\mathbb{R}^p \rightarrow \mathbb{R}$ is linear in $\rvw$.

For a constant tangent kernel, each element $K_{ii}(\rvw)$ is constant. Noting that $K_{ii}(\rvw) = \|\nabla_\rvw f(\rvw,\rvx_i)\|^2$, we have $\|\nabla f(\rvw,\rvx)\|$ is constant in $\rvw$, for all input $\rvx$.

The following arguments basically follow the idea from~\cite{868044} (a more general result  was shown in~\cite{sakai1996riemannian}). 


To simplify the notation, in the rest of the proof, we hide the argument $\rvx$, and we use $f(\rvw)$ to denote $f(\rvw;\rvx)$. 

Let $\|\nabla f(\rvw)\| = c$.  
Consider the ordinary differential equation (ODE) 
\begin{align*}
    \frac{d\rvw(t)}{dt} = \nabla f(\rvw(t)),
\end{align*}
where $\rvw(0) = \rvw_0 \in \mathbb{R}^p$ is the initial setting of the parameters.
We have \begin{align*}
    \frac{df}{dt} = \langle \nabla f,\frac{d\rvw}{dt}\rangle = c^2,
\end{align*}
and consequently
\begin{align}\label{eq:constantntk1}
   f(\rvw(t)) = c^2t+f(\rvw_0).
\end{align}
For any $t_1,t_2$, since $\|\nabla f(\rvw)\| = c$, we have
\begin{align*}
   c^2|t_1 - t_2|  = |f(\rvw(t_1)) - f(\rvw(t_2))| \leq c|\rvw(t_1) - \rvw(t_2)| ,
\end{align*}
but $| \rvw(t_1) - \rvw(t_2)| = |\int_{t_2}^{t_1} \|d\rvw(t)/dt\| dt| = c|t_1-t_2|$, which indicates
\begin{align}\label{eq:constantntk2}
    \rvw(t) = t \nabla f(\rvw_0) + \rvw_0.
\end{align}
And in the following we show for any $\rvv\in \mathbb{R}^p$, if $f(\rvv) = f(\rvw_0)$, we have 
\begin{align}\label{eq:constantntk3}
  \langle \nabla f(\rvw_0),\rvv-\rvw_0\rangle = 0.
\end{align}
Given $t \neq 0$, let $c:[0,1]\rightarrow \mathbb{R}^p$ be a differentiable curve joining $t\nabla f(\rvw_0) + \rvw_0$ and $\rvv$. By Eq.~(\ref{eq:constantntk1}) and Eq.~(\ref{eq:constantntk2}), we have
\begin{align*}
    c^2|t| = |f(\rvw(t)) - f(\rvw_0)| &= |f(\rvw_0+t\nabla f(\rvw_0)) - f(\rvw_0)| \\
    &= | \int_0^1 \langle \nabla f(c(s)),c'(s)\rangle ds | \\
    &\leq \int_0^1 \|c'(s)\| ds\\
    &= \| \rvv - t\nabla f(\rvw_0) - \rvw_0\|.
\end{align*}
It follows that
\begin{align*}
    c^4t^2 \leq \|\rvv-\rvw_0\|^2 + t^2 + 2t\langle \rvv-\rvw_0,\nabla f(\rvw_0)\rangle.
\end{align*}
Dividing by $t$ and taking $t$ to $\pm \infty$ allows us to have $\langle \nabla f(\rvw_0),\rvv-\rvw_0\rangle = 0$.\\
Then we construct the level set
\begin{align}\label{eq:levelset}
    M_a = M = \{\rvw \in \mathbb{R}^p : f(\rvw) = a\},
\end{align}
where $a\in \mathbb{R}$.
And its tangent space at $\rvw$ is
\begin{align}\label{eq:tangentspace}
    T_\rvw M = \{ \rvw + \rvv  \in \mathbb{R}^p : \langle \rvv,\nabla f(\rvw)\rangle = 0\}.
\end{align}
By Eq.(\ref{eq:constantntk3}) we have $\langle \rvv - \rvw, \nabla f(\rvw)\rangle $ for all $\rvv \in M$ that satisfies $f(\rvv) = f(\rvw)$. From Eq.(\ref{eq:tangentspace}) we can see $\rvv \in T_\rvw M$. Therefore $M\subset T_\rvw M$. By the fact that $M$ is a closed hypersurface, $M = T_\rvw M$ for all $\rvw \in M$.\\
Hence there exists a $\rvw'\in \mathbb{R}^p$ such that $\|\rvw'\| =1$ and the level set Eq.(\ref{eq:levelset}) is equivalently defined as $M_a = \{\frac{a}{c}\rvw'+\rvv': \langle \rvv',\rvw'\rangle = 0, \rvv'\in \mathbb{R}^p\}$ for all $a$. And we can construct a function $g:\mathbb{R}\rightarrow \mathbb{R}$ such that
\begin{align*}
    g(t) = f(\rvv'+t\rvw'),
\end{align*}
where $g'(t) = c$ for all $t$ which shows $f$ is linear.

\end{proof}
\section{Proof of Proposition~\ref{prop:Hessiantokernel}}\label{secapp:pfhessian}
{
\begin{proof}
Since the function $f$ is twice differentiable w.r.t. $\rvw$, according to Taylor's theorem, we have the following expression for the gradient:
\begin{equation}
    \nabla f(\rvw)  = \nabla f(\rvw_0) + \int_0^1 H(\rvw_0+t(\rvw-\rvw_0))(\rvw-\rvw_0)dt.
\end{equation}
Then the Euclidean norm of the gradient change is bounded by
\begin{eqnarray*}
    \|\nabla_{\rvw} f(\rvw) - \nabla_{\rvw} f(\rvw_0)\|&=& \|\int_0^1 H(\rvw_0+t(\rvw-\rvw_0))(\rvw-\rvw_0)dt\| \\
    &\le&\int_0^1\|H(\rvw_0+t(\rvw-\rvw_0))\|dt\cdot \|\rvw-\rvw_0\|. 
\end{eqnarray*}
Since $t\in[0,1]$ and the ball $B(\rvw_0,R)$ is convex, the point $\rvw_0+t(\rvw-\rvw_0)$ is within $B(\rvw_0,R)$. Hence,
\begin{equation}
    \|\nabla_{\rvw} f(\rvw) - \nabla_{\rvw} f(\rvw_0)\|\le \max_{\rvv\in B(\rvw_0,R)} \|H(\rvv)\|\cdot \|\rvw-\rvw_0\| \le \epsilon R.
\end{equation}
Hence, according to the definition of the tangent kernel, for any inputs $\rvx,\rvz\in \mathbb{R}^d$,
\begin{eqnarray}
    & &| K_{(\rvx,\rvz)}(\rvw) - K_{(\rvx,\rvz)}(\rvw_0)| \nonumber \\
    &\le& \|\nabla_{\rvw} f(\rvw;\rvx) - \nabla_{\rvw} f(\rvw_0;\rvx)\|\cdot \|\nabla_{\rvw} f(\rvw;\rvz)\| + \|\nabla_{\rvw} f(\rvw;\rvz) -  \nabla_{\rvw} f(\rvw_0;\rvz)\|\cdot \|\nabla f(\rvw_0;\rvx)\|\nonumber\\
    &\le& \epsilon R (\|\nabla_{\rvw} f(\rvw_0;\rvx)\| + \|\nabla_{\rvw} f(\rvw;\rvz)\|).\nonumber
\end{eqnarray}
Since $f$ is smooth, the gradients $\nabla_{\rvw} f(\rvw_0)$ and $\nabla_{\rvw} f(\rvw)$ are bounded. Therefore, $| K_{(\rvx,\rvz)}(\rvw) - K_{(\rvx,\rvz)}(\rvw_0)| = O(\epsilon R)$.
\end{proof}
}

\section{Proof of Theorem~\ref{thm:hessian_infinity}}\label{secapp:thm_general}
\begin{proof}
The Hessian matrix $H$ of the neural network can be written as the following structure:
\begin{equation}\label{eq:hessianmatrix}
    H = \left(\begin{array}{cccc}
    H^{(1,1)} & H^{(1,2)} & \cdots & H^{(1,L+1)}\\
    H^{(2,1)} & H^{(2,2)} & \cdots & H^{(2,L+1)}\\
    \vdots & \vdots & \ddots & \vdots \\
    H^{(L+1,1)} & H^{(L+1,2)} & \cdots & H^{(L+1,L+1)}
    \end{array}
    \right).
\end{equation}
Here, each Hessian block $H^{(l_1,l_2)}:= \frac{\partial^2 f}{\partial \rvw^{(l_1)}\partial \rvw^{(l_2)}}$ is the second derivative of $f$ w.r.t. its weights of $l_1$-th and $l_2$-th layers, where we  treat the final layer parameters $\rvv$ as $\rvw^{(L+1)}$.

The following lemma allows us to bound the Hessian spectral norm by the norms of its blocks (see proof in Appendix~\ref{secapp:techincal1}).
\begin{lemma}\label{lemma:combine}
Spectral norm of a matrix $H$ (\ref{eq:hessianmatrix}) is upper bounded by the sum of the spectral norm of its blocks, i.e. $\| H \| \leq \sum_{l_1,l_2} \| H^{(l_1,l_2)}\|$, $l_1,l_2 \in [L+1]$.
\end{lemma}

Now, we analyze the Hessian blocks case by case.
Since the Hessian matrix is symmetry, without loss of generosity, we assume $1\le l_1 \le l_2 \le L+1$.

\paragraph{Case 1: $1\le l_1 \le l_2 \le L$.}
By the chain rule, the gradient of the model $f$ w.r.t. the weights of layer $l$, can be written as
\begin{equation}\label{eq:generalgrad}
    \frac{\partial f}{\partial \rvw^{(l)}} = \frac{\partial \alpha^{(l)}}{\partial \rvw^{(l)}}\left(\prod_{l' = l+1}^L \frac{\partial \alpha^{(l')}}{\partial \alpha^{(l'-1)}}\right) \frac{1}{\sqrt{m}}\rvv.
\end{equation}
Then, the Hessian block has the following expression:
\begin{eqnarray}
    & &H^{(l_1,l_2)}\nonumber\\ &=& \frac{\partial^2 \alpha^{(l_1)}}{(\partial \rvw^{(l_1)})^2}\frac{\partial f}{\partial \alpha_{l_1}}  \cdot \mathbb{I}_{l_1 = l_2}+\left(\frac{\partial \alpha^{(l_1)}}{\partial \rvw^{(l_1)}}\prod_{l' = l_1+1}^{l_2-1} \frac{\partial \alpha^{(l')}}{\partial \alpha^{(l'-1)}}\right) \frac{\partial^2 \alpha^{(l_2)}}{\partial \alpha^{(l_2-1)}\partial \rvw^{(l_2)}} \left( \frac{\partial f}{\partial \alpha^{(l_2)}}\right)\nonumber\\
    & &+\sum_{l=l_2+1}^L\left(\frac{\partial \alpha^{(l_1)}}{\partial \rvw^{(l_1)}}\prod_{l' = l_1+1}^{l-1} \frac{\partial \alpha^{(l')}}{\partial \alpha^{(l'-1)}}\right) \frac{\partial^2 \alpha^{(l)}}{(\partial \alpha^{(l-1)})^2}  \left(\frac{\partial \alpha^{(l_2)}}{\partial \rvw^{(l_2)}}\prod_{l' = l_2+1}^{l} \frac{\partial \alpha^{(l')}}{\partial \alpha^{(l'-1)}}\right)\left( \frac{\partial f}{\partial \alpha^{(l)}}\right)\nonumber
\end{eqnarray}
Hence, the spectral norm of Hessian block $H^{(l_1,l_2)}$ is bounded by
\begin{eqnarray}
   & & \left\|H^{(l_1,l_2)}\right\| \nonumber \\
   &\le& \left\|\frac{\partial^2 \alpha^{(l_1)}}{(\partial \rvw^{(l_1)})^2}\right\|_{2,2,1} \left\|\frac{\partial f}{\partial \alpha^{(l_1)}}\right\|_{\infty} + \left\|\frac{\partial \alpha^{(l_1)}}{\partial \rvw^{(l_1)}}\right\|\prod_{l' = l_1+1}^{l_2-1} \left\|\frac{\partial \alpha^{(l')}}{\partial \alpha^{(l'-1)}}\right\| \left\|\frac{\partial^2 \alpha^{(l_2)}}{\partial \alpha^{(l_2-1)}\partial \rvw^{(l_2)}}\right\|_{2,2,1}  \left\|\frac{\partial f}{\partial \alpha^{(l_2)}}\right\|_{\infty}\nonumber\\
   & &+\sum_{l=l_2+1}^L\left\|\frac{\partial \alpha^{(l_1)}}{\partial \rvw^{(l_1)}}\right\|\prod_{l' = l_1+1}^l \left\|\frac{\partial \alpha^{(l')}}{\partial \alpha^{(l'-1)}}\right\| \left\|\frac{\partial^2 \alpha^{(l)}}{(\partial \alpha^{(l-1)})^2}\right\|_{2,2,1}  \left\|\frac{\partial\alpha^{(l_2)}}{\partial \rvw^{(l_2)}}\right\|\prod_{l' = l_2+1}^{l}\left\| \frac{\partial \alpha^{(l')}}{\partial \alpha^{(l'-1)}}\right\| \left\|\frac{\partial f}{\partial \alpha^{(l)}}\right\|_{\infty}\nonumber\\
   &\le&
   \left\|\frac{\partial^2 \alpha^{(l_1)}}{(\partial \rvw^{(l_1)})^2}\right\|_{2,2,1} \left\|\frac{\partial f}{\partial \alpha^{(l_1)}}\right\|_{\infty} + \mathsf{L}_{\phi}^{l_2-l_1-1} \left\|\frac{\partial \alpha^{(l_1)}}{\partial \rvw^{(l_1)}}\right\|\left\|\frac{\partial^2 \alpha^{(l_2)}}{\partial \alpha^{(l_2-1)}\partial \rvw^{(l_2)}}\right\|_{2,2,1}  \left\|\frac{\partial f}{\partial \alpha^{(l_2)}}\right\|_{\infty}\nonumber\\
   & &+\sum_{l=l_2+1}^L\mathsf{L}_{\phi}^{2l-l_1-l_2}\left\|\frac{\partial \alpha^{(l_1)}}{\partial \rvw^{(l_1)}}\right\| \left\|\frac{\partial^2 \alpha^{(l)}}{(\partial \alpha^{(l-1)})^2}\right\|_{2,2,1}  \left\|\frac{\partial\alpha^{(l_2)}}{\partial \rvw^{(l_2)}}\right\| \left\|\frac{\partial f}{\partial \alpha^{(l)}}\right\|_{\infty}.\nonumber
\end{eqnarray}
By the definitions in Eq.(\ref{eq:defi_quantities}), we have 
\begin{equation}
   \left\|H^{(l_1,l_2)}\right\| \le C'_1 \mathcal{Q}_{2,2,1}(f)\mathcal{Q}_\infty(f),
\end{equation}
with $C'_1 = L^2\mathsf{L}_{\phi}^{2L}+L\mathsf{L}_{\phi}^L+1$.



\paragraph{Case 2: $1\le l_1 < l_2 = L+1$.}
Using the gradient expression in Eq.(\ref{eq:generalgrad}), we have
\begin{equation}
    H^{(l_1,L+1)} = \frac{1}{\sqrt{m}}\frac{\partial \alpha^{(l_1)}}{\partial \rvw^{(l_1)}}\left(\prod_{l' = l_1+1}^L \frac{\partial \alpha^{(l')}}{\partial \alpha^{(l'-1)}}\right).
\end{equation}
Hence, 
\begin{equation}
    \|H^{(l_1,L+1)}\| \le  \frac{1}{\sqrt{m}}\left\|\frac{\partial \alpha^{(l_1)}}{\partial \rvw^{(l_1)}}\right\|\prod_{l' = l_1+1}^L \left\|\frac{\partial \alpha^{(l')}}{\partial \alpha^{(l'-1)}}\right\| \le  \frac{1}{\sqrt{m}}\mathsf{L}_{\phi}^{L}\mathcal{Q}_L(f).
\end{equation}
\paragraph{Case 3: $l_1 = l_2 = L+1$.}
In this case, the Hessian block $H^{(L+1,L+1)}$ is simply zero. Hence, the spectral norm is zero.

\smallskip
Applying Lemma~\ref{lemma:combine}, we immediately obtain the desired result.
\end{proof}

\section{Proof for Lemma~\ref{lemma:fcn_quantities}}\label{secapp:fcn_quantities}
According to the definitions of the quantities $\mathcal{Q}_{\infty}(f)$, $\mathcal{Q}_{2,2,1}(f)$ and $\mathcal{Q}_{L}(f)$ in Eq.(\ref{eq:defi_quantities}), it suffices to show that the followings layer-wise properties hold everywhere in the ball $B(\rmW_0,R)$ with high probability over the initialization: 
\begin{itemize}
    \item The vector $\infty$-norm $\left\|\frac{\partial{f}}{\partial\alpha^{(l)}}\right\|_\infty = \tilde{O}(1/\sqrt{m})$, for all $l\in [L]$;
    \item The matrix spectral norm $\left\|\frac{\partial \alpha^{(l)}}{\partial \rvw^{(l)}}\right\| = O(1)$ w.r.t. $m$, for all $l\in [L]$;
    \item The $(2,2,1)$-norms of order $3$ tensors, $\left\|\frac{\partial^2\alpha^{(l)}}{\partial\rvw^{(l)2}}\right\|_{2,2,1}$, $\left\|\frac{\partial^2\alpha^{(l)}}{\partial \alpha^{(l-1)}\partial \rvw^{(l)}}\right\|_{2,2,1}$ and $\left\|\frac{\partial^2\alpha^{(l)}}{(\partial\alpha^{(l-1)})^2}\right\|_{2,2,1}$ are all of the order $O(1)$ w.r.t. $m$, for all $l\in [L]$.
\end{itemize}
We start the proof with some preliminary results, and then prove the above statements one by one.

\subsection{Preliminaries}

 
The fully connected neural network is defined in the following way:
\begin{align}\label{eq:fcnet}
 &\alpha^{(0)} = \rvx, \nonumber\\
 &\alpha^{(l)}= \sigma(\tilde{\alpha}^{(l)}), \ \  \tilde{\alpha}^{(l)} = \frac{1}{\sqrt{m_{l-1}}}W^{(l)}\alpha^{(l-1)},\ \forall l \in [L] \nonumber \\
 &f = \frac{1}{\sqrt{m}}\rvv^T \alpha^{(L)},
\end{align}
where $m_0 =d$ which is the dimension of the input $\rvx$, and $m_l = m$ for all $l\in [L]$. The trainable parameters of this network are $\rmW := \{W^{(1)},W^{(2)},\cdots,W^{(L)},W^{(L+1)}:=\rvv\}$, and are initialized by the random Gaussian initialization,  i.e., each parameter $(W^{(l)}_0)_{ij}\sim \mathcal{N}(0,1), \forall l\in [L]$, and $v_{0,i} \sim \mathcal{N}(0,1)$, $i,j\in [m]$. As the parameters $W^{(l)}$ of each layer are reshaped into matrices, the Euclidean norm of parameters becomes $\|\rmW\| := (\sum_{l=1}^{L+1}\|W^{(l)}\|_F^2)^{1/2}$, where $\|\cdot\|_F$ is the Frobenius norm of a matrix.
 
To make the presentation of the proof as simple as possible, we first make the following assumption about the initial parameters $\rmW_0$. Then we prove it in Lemma~\ref{lemma:gaussl2} that the assumption is satisfied  with high probability over the random Gaussian initialization.
 
 \begin{assumption}\label{assu:network1}
We assume that there exists a constant $c_0>0$ such that, for all initial weight matrices/vector $W_0^{(l)}$,  $\|W_0^{(l)}\| \le c_0\sqrt{m}$, where $l\in[L+1]$. 
\end{assumption} 

\begin{lemma}[Spectral norms of initial weight matrices]\label{lemma:gaussl2}
If the parameters are initialized as $(W^{(l)}_0)_{ij} \sim \mathcal{N}(0,1)$ for all $l\in[L+1]$ and $m>d$, then, for each layer $l\in [L+1]$, we have with probability at least $1-2\exp(-\frac{m}{2})$,
\begin{equation}\label{eq:assump1pf}
    \|W_0^{(l)}\| \le 3\sqrt{m}.
\end{equation}
\end{lemma}
The proof is in Appendix~\ref{secapp:gaussian}. 

We further assume that, for the input $\rvx\in\mathbb{R}^d$, each component is bounded, i.e. $|x_i| \le C_{\rvx}$, for some constant $C_{\rvx}$ and for all $i\in [d]$. This assumption covers most of the practical cases.


We prove the following lemma which states that the norm of the matrix $W^{(l)}$ keeps its order in a finite ball around the $W_0^{(l)}$ .
\begin{lemma}\label{lemma:wl2}
If  $\rmW_0$ satisfies Assumption~\ref{assu:network1}, then for any $\rmW$ such that $\|\rmW-\rmW_0\| \le R$, we have
\begin{equation}
    \|W^{(l)}\| \le c_0\sqrt{m} + R= O(\sqrt{m}), \ \forall l\in [L+1].
\end{equation}
\end{lemma}
See the proof in Appendix~\ref{secapp:weight}. The following lemma gives bounds on the Euclidean norm of the vector of hidden neurons for each layer.

\begin{lemma}\label{lemma:activations}
If  $\rmW_0$  satisfies Assumption~\ref{assu:network1}, then, for any  $\rmW$ such that $\|\rmW-\rmW_0\| \le R$, we have, at all hidden layers
\begin{equation}
    \|{\alpha}^{(l)}(\rmW)\| \le \mathsf{L}_{\sigma}^{l}(c_0+R/\sqrt{m})^{l}\sqrt{m}C_{\rvx} + \sum_{i=1}^l \mathsf{L}_{\sigma}^{i-1}(c_0+R/\sqrt{m})^{i-1}\sigma(0) = O(\sqrt{m}),\ \forall l \in [L].\label{eqapp:activations}
\end{equation}
Particularly, for the input layer,
\begin{align}
    \|\alpha^{(0)}\| = \|\rvx\|  \leq \sqrt{d}C_\rvx= O(1).
\end{align}
\end{lemma}
The proof is in Appendix \ref{secapp:activations} . 



\subsection{Matrix spectral norm $\left\|{\partial \alpha^{(l)}}/{\partial \rvw^{(l)}}\right\| = O(1)$ and Lipschitz continuity of $\alpha^{(l)}$ w.r.t $\alpha^{(l-1)}$}\label{secapp:lipschitz}
Here, we show that, for any $l\in[L]$ and at any point $\rmW\in B(\rmW_0,R)$, both $\left\|{\partial \alpha^{(l)}}/{\partial \rvw^{(l)}}\right\|$ and $\left\|{\partial \alpha^{(l)}}/{\partial \alpha^{(l-1)}}\right\|$ are of the order $O(1)$, with high probability over the random Gaussian initialization of $\rmW_0$, the latter of which is essentially the  Lipschitz continuity of $\alpha^{(l)}$ w.r.t $\alpha^{(l-1)}$.

{\bf When $l = 2,3,\cdots,L$.}
Recall from Eq.(\ref{eq:fcnlayer}) that, a fully connected layer $\alpha^{(l)}$ is defined as, for $l = 2,3,\cdots,L$:\\
\begin{align}\label{eq:fcphi}
   \alpha^{(l)} 
   = \sigma\left(\frac{1}{\sqrt{m}}W^{(l)}\alpha^{(l-1)}\right).
\end{align}
The term $\tilde{\alpha}^{(l)}:=\frac{1}{\sqrt{m}}W^{(l)}\alpha^{(l-1)}$ is also known as preactivation.

Note that, in this case,  the parameter vector $\rvw^{(l)}$ is reshaped to an $m\times m$ matrix $W^{(l)}$.
The first derivatives of $\alpha^{(l)}$  are
\begin{eqnarray}
\left(\frac{\partial \alpha^{(l)}}{\partial \alpha^{(l-1)}}\right)_{i,j}&=& \frac{1}{\sqrt{m}}\sigma'(\tilde{\alpha}^{(l)}_i) W^{(l)}_{ij},\\
\left(\frac{\partial \alpha^{(l)}}{\partial W^{(l)}}\right)_{i,jj'} &=& \frac{1}{\sqrt{m}}\sigma'(\tilde{\alpha}^{(l)}_i)\alpha^{(l-1)}_{j'}\mathbb{I}_{i=j}.
\end{eqnarray}
By the definition of spectral norm, $\|A\| = \sup_{\|\rvv\|=1}\|A\rvv\|$, we have, for all $2\leq l \leq L$,
\begin{eqnarray*}
\left\|\frac{\partial \alpha^{(l)}}{\partial \alpha^{(l-1)}}\right\|^2 &=&\sup_{\|\rvv\|=1}\frac{1}{{m}}\sum_{i=1}^m\left(\sigma'(\tilde{\alpha}^{(l)}_i) W^{(l)}_{ij} v_j\right)^2\\
&=&\sup_{\|\rvv\|=1}\frac{1}{{m}} \|{\Sigma'}^{(l)}W^{(l)}\rvv \|^2\\
&\le& \frac{1}{{m}}\| {\Sigma'}^{(l)}\|^2\|W^{(l)}\|^2\\
&\le& \mathsf{L}_\sigma^2 (c_0 + R/\sqrt{m})^2 = O(1),
\end{eqnarray*}
where ${\Sigma'}^{(l)}$ is a diagonal matrix, with the diagonal entry ${\Sigma'}^{(l)}_{ii}=\sigma'(\tilde{\alpha}^{(l)}_i)$. In the last inequality above, we used Lemma~\ref{lemma:wl2} and the Lipschitz continuity of the activation $\sigma(\cdot)$.

Similarly, we have
\begin{eqnarray}
\left\|\frac{\partial \alpha^{(l)}}{\partial W^{(l)}}\right\|^2 &=& \sup_{\|V\|_F = 1} \frac{1}{{m}}\sum_{i=1}^m\Big(\sum_{j,j'}\sigma'(\tilde{\alpha}^{(l)}_i)\alpha^{(l-1)}_{j'}\mathbb{I}_{i=j}V_{jj'}\Big)^2\nonumber\\
&=&\sup_{\|V\|_F = 1} \frac{1}{{m}}\| {\Sigma'}^{(l)} V \alpha^{(l-1)}\|^2\nonumber \\
    &\leq& \frac{1}{{m}}\| {\Sigma'}^{(l)}\|^2 \| \alpha^{(l-1)}\|^2 \nonumber\\
    &\leq& \left(\mathsf{L}_\sigma^{l}(c_0+R)^{l-1}C_\rvx\right)^2  =O(1).\label{eqaux:nabla_w}
\end{eqnarray}
In the last inequality, we used Lemma \ref{lemma:activations} and the Lipschitz continuity of the activation $\sigma(\cdot)$.

{\bf When $l=1$.} The layer function is:
\begin{align}
    \alpha^{(1)} = \phi_1(W^{(1)};\alpha^{(0)}) = \sigma\left(\frac{1}{\sqrt{d}}W^{(1)}\rvx\right).
\end{align}
In this layer, the input $\rvx$ is fixed (independent of trainable parameters) and not a dynamical variable. Hence, $\partial \alpha^{(1)}/\partial \rvx$ is not an interesting object in our Hessian analysis\footnote{Indeed, it does not show up in the Hessian analysis (c.f. the proof of Theorem~\ref{thm:hessian_infinity} in Section \ref{secapp:thm_general}).}.

For $\partial \alpha^{(1)}/\partial W^{(1)}$, we have (with a similar analysis as in Eq.(\ref{eqaux:nabla_w})),
\begin{eqnarray*}
\|\partial \alpha^{(1)}/\partial W^{(1)}\|^2 &\le& \frac{1}{{d}}\| {\Sigma'}^{(l)}\|^2 \| \rvx\|^2 \le \mathsf{L}_{\sigma}^2C_{\rvx}^2 = O(1).
\end{eqnarray*}


\subsection{$(2,2,1)$-norms of order $3$ tensors are $O(1)$}\label{secapp:221-norms-fcn}
\begin{proof}
We consider the first layer i.e. $l=1$ and the rest of the layers i.e.  $l=2,3,\cdots,L$ separately.

{\bf When $l=2,3,\cdots,L$.} The second derivatives of the vector-valued layer function $\alpha^{(l)}$, which are order $3$ tensors, have the following expressions:
\begin{eqnarray}
 & &\left(\frac{\partial^2 \alpha^{(l)}}{(\partial \alpha^{(l-1)})^2}\right)_{i,j,k} = \frac{1}{{m}}\sigma''(\tilde{\alpha}^{(l)}_i)W_{ij}^{(l)}W_{ik}^{(l)},\\
 & &\left(\frac{\partial^2 \alpha^{(l)}}{\partial \alpha^{(l-1)}\partial W^{(l)}}\right)_{i,j,kk'} = \frac{1}{{m}}\sigma''(\tilde{\alpha}^{(l)}_i)W_{ij}^{(l)}\alpha_{k'}^{(l-1)}\mathbb{I}_{i=k},\\
 & &\left(\frac{\partial^2 \alpha^{(l)}}{(\partial W^{(l)})^2}\right)_{i,jj',kk'} =\frac{1}{{m}}\sigma''(\tilde{\alpha}^{(l)}_i)\alpha_{j'}^{(l-1)}\alpha_{k'}^{(l-1)}\mathbb{I}_{i=k=j}.
\end{eqnarray}

By the definition of the $(2,2,1)$-norm for order $3$ tensors, and  Lemma~\ref{lemma:wl2}, we get
 \begin{align}
    \left\|\frac{\partial^2 \alpha^{(l)}}{(\partial \alpha^{(l-1)})^2}\right\|_{2,2,1} 
    &=\sup_{\|\rvv_1\| = \|\rvv_2\|=1}\frac{1}{m}\sum_{i=1}^{m}\left|\sigma''(\tilde{\alpha}^{(l)}_i)(W^{(l)}\rvv_1)_i(W^{(l)}\rvv_2)_i\right|\nonumber \\
    &\leq \sup_{\|\rvv_1\| = \|\rvv_2\|=1}\frac{1}{m}\beta_\sigma \sum_{i=1}^{m}\left|(W^{(l)}\rvv_1)_i(W^{(l)}\rvv_2)_i\right|\nonumber\\
    &\leq \sup_{\|\rvv_1\| = \|\rvv_2\|=1}\frac{1}{2m}\beta_\sigma \sum_{i=1}^{m}(W^{(l)}\rvv_1)_i^2 +(W^{(l)}\rvv_2)_i^2\nonumber\\
    &\leq \frac{1}{2m}\beta_\sigma \sup_{\|\rvv_1\| = \|\rvv_2\|=1}(\| W^{(l)}\rvv_1\|^2 +  \|  W^{(l)}\rvv_2\|^2 )\nonumber \\
    &\le  \frac{1}{2m}\beta_\sigma (\| W^{(l)}\|^2 +  \|  W^{(l)}\|^2 ) \nonumber\\
    &\leq {\beta_\sigma (c_0+R/\sqrt{m})^2} = O(1).\label{eq:fcn_order_3}
\end{align}
Similarly,  by using Lemma~\ref{lemma:wl2} and Lemma~\ref{lemma:activations}, we have,
\begin{align*}
        \left\|\frac{\partial^2 \alpha^{(l)}}{\partial \alpha^{(l-1)}\partial W^{(l)}}\right\|_{2,2,1} 
    &=\sup_{\|\rvv_1\| = \|V_2\|_F=1}\frac{1}{m}\sum_{i=1}^m\left|\sigma''(\tilde{\alpha}^{(l)}_i)(W^{(l)}\rvv_1)_i(V_2\alpha^{(l)})_i\right| \\
    &\leq \sup_{\|\rvv_1\| = \|V_2\|_F=1}\frac{1}{2m}\beta_\sigma (\|W^{(l)}\rvv_1\|^2+\|V_2\alpha^{(l-1)}\|^2)  \\
    &\leq \frac{1}{2m}\beta_\sigma (\| W^{(l)}\|^2   +\|\alpha^{(l-1)}\|^2) \\
    &\leq \frac{\beta_\sigma}{2}(c_0+R/\sqrt{m})^2 + \frac{\beta_\sigma}{2} \mathsf{L}_\sigma^{2l-2}(c_0+R/\sqrt{m})^{(2l-2)} C^2_\rvx = O(1).
\end{align*}
And
\begin{align}
     \left\|\frac{\partial^2 \alpha^{(l)}}{(\partial W^{(l)})^2}\right\|_{2,2,1} 
    &=\sup_{\|V_1\|_F = \|V_2\|_F=1}\frac{1}{m}\sum_{i=1}^m\left|\sigma''(\tilde{\alpha}^{(l)}_i)(V_1\alpha^{(l-1)})_i(V_2\alpha^{(l-1)})_i\right| \nonumber  \\
    &\leq \sup_{\|V_1\|_F = \|V_2\|_F=1}\frac{1}{2m}\beta_\sigma(\|V_1\alpha^{(l-1)}\|^2+ \|V_2\alpha^{(l-1)}\|^2) \nonumber  \\
    &\leq \frac{1}{2m}\beta_\sigma(\|\alpha^{(l-1)}\|^2 +  \|\alpha^{(l-1)}\|^2) \nonumber  \\
    &\leq \beta_\sigma \mathsf{L}_\sigma^{2l-2}(c_0+R/\sqrt{m})^{2l-2}C_\rvx^2 = O(1).\label{eqaux:partial_w_w}
\end{align}

{\bf When $l=1$.} As discussed in Section~\ref{secapp:lipschitz}, the input $\alpha^{(0)}=\rvx$ is constant, we only need to analyze the tensor $\frac{\partial^2 \alpha^{(l)}}{(\partial W^{(l)})^2}$ in this case.
With a similar analysis in Eq.(\ref{eqaux:partial_w_w}), we have
\begin{equation}
    \left\|\frac{\partial^2 \alpha^{(1)}}{(\partial W^{(1)})^2}\right\|_{2,2,1} \le \frac{1}{2d}\beta_\sigma(\|\alpha^{(0)}\|^2 +  \|\alpha^{(0)}\|^2) \le \beta_{\sigma}C_{\rvx}^2 = O(1).
\end{equation}
 \end{proof}

\subsection{Vector $\infty$-norm is $\tilde{O}(1/\sqrt{m})$}\label{secapp:infinity-norm-fcn}

 \begin{proof}
 First of all, we present a few useful facts,
 Lemma~\ref{lemma:everyactivation}-\ref{lemma:everyderivatives} 
 that will be used during the proof. The proofs of the following lemmas are in
 Appendix~\ref{secapp:everyactivation}-\ref{secapp:everyderivatives}. 
 
We first show that each activation of the hidden layers is bounded at initialization,  with high probability.
 \begin{lemma}\label{lemma:everyactivation}
For any $l \in [L]$, given $i\in [m]$, with probability at least $1 - 2e^{-{c_\alpha^{(l)}\ln^2(m)}}$ for some constant $c_\alpha^{(l)}>0$, $|\alpha_{i}^{(l)}| = \tilde{O}(1)$ at initialization.
\end{lemma}
Define vector $\rvb^{(l)} :=\partial f/\partial {\alpha^{(l)}} \in \mathbb{R}^m$ for $l\in [L]$. And we use $\rvb_0$ to denote $\rvb$ at initialization. Specifically, $\rvb^{(l)}$ takes the following form:
\begin{align}\label{eq:derivatives_expr}
    \rvb^{(l)} = \prod_{l'=l+1}^L\left(\frac{1}{\sqrt{m}}(W^{(l')})^T{\Sigma'}^{(l')}\right)\frac{1}{\sqrt{m}}\rvv,
\end{align}
where ${\Sigma'}^{(l')}$ is a diagonal matrix, with $({\Sigma'}^{(l')})_{ii} = \sigma'(\tilde{\alpha}_i^{(l')})$.

The following lemma gives an upper bound to Euclidean norms of $\rvb^{(l)}$ in the ball $B(\rmW_0,R)$.
\begin{lemma}\label{lemma:derivatives}
If the initial parameters $\rmW_0$ of the multi-layer neural network $f(\rmW)$ satisfies Assumption~\ref{assu:network1}, then, for any  $\rmW$ such that $\|\rmW-\rmW_0\| \le R$, we have, at all hidden layers, i.e., 
 $\forall l\in [L]$,  
 \begin{equation}
 \|\rvb^{(l)}\| \le \mathsf{L}_\sigma^{L-l}(c_0 + R/\sqrt{m})^{L-l+1}.
 \end{equation} 
In particular, at initialization, 
\begin{equation}\label{eqapp:rvb0-2norm}
    \|\rvb^{(l)}_0\| \le \mathsf{L}_\sigma^{L-l}c_0^{L-l+1}.
\end{equation}
\end{lemma}
We proceed to show  all the components of $\rvb^{(l)}_0$ are of order  $\tilde{O}(\frac{1}{\sqrt{m}})$ with high probability.
\begin{lemma}\label{lemma:everyderivatives}
 With probability at least $1-me^{-c_b^{(l)}\ln^2(m)}$ for some constant $c_b^{(l)} >0$, $\|\rvb_0^{(l)}\|_\infty = \tilde{O}(1/\sqrt{m})$.
\end{lemma}


Now we show besides at initialization, $\|\rvb\|_\infty$ is of order $\tilde{O}(\frac{1}{\sqrt{m}})$ in the ball $B(\rmW_0,R)$ with high probability. Technically, we bound the difference of the $\infty$-norm by the difference of $2$-norm.

First of all, we prove, by induction, the following claim: for all $l\in [L]$,
\begin{equation}
    \|\rvb^{(l)}-\rvb^{(l)}_0\| = \tilde{O}\left(\frac{1}{\sqrt{m}}\right).\label{eqapp:inductionb}
\end{equation}
In the base case, we consider $l=L$. We have 
\begin{eqnarray}
 \rvb^{(L)} &=& \frac{1}{\sqrt{m}}\rvv, \nonumber\\
 \rvb_0^{(L)} &=& \frac{1}{\sqrt{m}}\rvv_0. \nonumber
\end{eqnarray}
Hence, 
\begin{equation}
    \|\rvb^{(L)}-\rvb_0^{(L)}\|  = \frac{1}{\sqrt{m}}\|\rvv-\rvv_0\| \le \frac{1}{\sqrt{m}}\|\rmW-\rmW_0\| \le \frac{1}{\sqrt{m}}R.
\end{equation}
Now, suppose that $\|\rvb^{(l)} -  \rvb_0^{(l)}\|  = \tilde{O}(\frac{1}{\sqrt{m}})$. Then
\begin{align}
   \left\| \rvb^{(l-1)} - \rvb_0^{(l-1)}\right\|
   &= \frac{1}{\sqrt{m}}\left\|(W^{(l)})^T\Sigma^{\prime(l)}\rvb^{(l)} -(W_0^{(l)})^T\Sigma_0^{\prime(l)}\rvb_0^{(l)} + (W^{(l)}_0)^T\Sigma^{\prime(l)}\rvb_0^{(l)} \right.\nonumber\\
    &\indent + \left.(W^{(l)}_0)^T\Sigma^{\prime(l)}\rvb^{(l)} - (W^{(l)}_0)^T\Sigma^{\prime(l)}\rvb_0^{(l)}  - (W^{(l)}_0)^T\Sigma^{\prime(l)}\rvb^{(l)}\right \| \nonumber\\
    &=\frac{1}{\sqrt{m}} \left\|\left((W^{(l)})^T - (W_0^{(l)})^T \right)\Sigma^{\prime(l)}\rvb^{(l)} + (W_0^{(l)})^T\left(\Sigma^{\prime(l)} - \Sigma_0^{\prime(l)}\right)\rvb_0^{(l)}\right. \nonumber\\
    &\indent +\left. (W^{(l)}_0)^T\Sigma^{\prime(l)}\left(\rvb^{(l)} -  \rvb_0^{(l)}\right)\right\|\nonumber \\
    &\leq \frac{1}{\sqrt{m}}\left\|W^{(l)} - W_0^{(l)} \right\|_2\|\Sigma^{\prime(l)}\|\|\rvb^{(l)}\| + \frac{1}{\sqrt{m}}\|W_0^{(l)}\|\left\|\left(\Sigma^{\prime(l)} - \Sigma_0^{\prime(l)}\right)\rvb_0^{(l)} \right\| \nonumber\\
    &\indent +\frac{1}{\sqrt{m}} \|W^{(l)}_0\|\|\Sigma^{\prime(l)}\|\left\|\rvb^{(l)} -  \rvb_0^{(l)}\right\|,\label{eqapp:bdiffinterm}
\end{align}
where ${\Sigma'}^{(l)}$ is a diagonal matrix, with $({\Sigma'}^{(l)})_{ii} = \sigma'(\tilde{\alpha}_i^{(l)})$.

To bound the second additive term above, we need the following inequality:
\begin{eqnarray}
 & &\|\tilde{\alpha}^{(l)}(\mathbf{W}) - \tilde{\alpha}^{(l)}(\mathbf{W}_0) \|\nonumber\\
 &=& \left\|\frac{1}{\sqrt{m}}W^{(l)}{\alpha}^{(l-1)}(\mathbf{W}) - \frac{1}{\sqrt{m}}W^{(l)}_0{\alpha}^{(l-1)}(\mathbf{W}_0)\right\|\nonumber\\
 &\le&\frac{1}{\sqrt{m}}\|W^{(l)}_0\|\cdot \mathsf{L}_{\sigma}\cdot \|\tilde{\alpha}^{(l-1)}(\mathbf{W}) - \tilde{\alpha}^{(l-1)}(\mathbf{W}_0) \| +  \frac{1}{\sqrt{m}}\|W^{(l)}-W^{(l)}_0\|\|{\alpha}^{(l-1)}(\mathbf{W})\| \nonumber\\
 &\le&  c_0\mathsf{L}_\sigma \|\tilde{\alpha}^{(l-1)}(\mathbf{W}) - \tilde{\alpha}^{(l-1)}(\mathbf{W}_0)+\frac{1}{\sqrt{m}}\|W^{(l)}-W^{(l)}_0\|\|{\alpha}^{(l-1)}(\mathbf{W})\| \nonumber\\
 &=& c_0\mathsf{L}_{\sigma}\|\tilde{\alpha}^{(l-1)}(\mathbf{W}) - \tilde{\alpha}^{(l-1)}(\mathbf{W}_0) \| + O(1),\nonumber
\end{eqnarray}
where the last equality is the result of Lemma~\ref{lemma:activations} that ${\|\alpha^{(l-1)}}\| = O(\sqrt{m})$.

Recursively applying the above equation, since $\|\tilde{\alpha}^{(1)}(\mathbf{W}) - \tilde{\alpha}^{(1)}(\mathbf{W}_0) \|\leq \frac{R}{\sqrt{d}}C_\rvx$,  we have
\begin{eqnarray}
 \|\tilde{\alpha}^{(l)}(\mathbf{W}) - \tilde{\alpha}^{(l)}(\mathbf{W}_0) \| 
&=& c_0^{l-1} \mathsf{L}_\sigma^{l-1}\|\tilde{\alpha}^{(1)}(\mathbf{W}) - \tilde{\alpha}^{(1)}(\mathbf{W}_0) \| + O(1)  = O(1).
\label{eqapp:tildealphachange}
\end{eqnarray}
Also, note that $\Sigma'$ is a diagonal matrix, then,  we have
\begin{align}
 \left\|\left[\Sigma^{\prime(l)} - \Sigma_0^{\prime(l)}\right]\rvb_0^{(l)} \right\| &= \sqrt{\sum_{i=1}^m (\rvb^{(l)}_{0})_i^2\left(\sigma'(\tilde{\alpha}_i^{(l)}(\mathbf{W})) - \sigma'(\tilde{\alpha}_i^{(l)}(\mathbf{W}_0)) \right)^2} \nonumber\\
 &\leq \|\rvb_0^{(l)}\|_\infty \sqrt{\sum_{i=1}^m \left[\sigma'(\tilde{\alpha}_i^{(l)}(\mathbf{W})) - \sigma'(\tilde{\alpha}_i^{(l)}(\mathbf{W}_0)) \right]^2} \nonumber\\
 &\leq \|\rvb_0^{(l)}\|_\infty \cdot \beta_\sigma  \|\tilde{\alpha}^{(l)}(\mathbf{W}) - \tilde{\alpha}^{(l)}(\mathbf{W}_0) \|
 = \tilde{O}\left(\frac{1}{\sqrt{m}}\right),\label{eqapp:11}
\end{align}
where we used Lemma~\ref{lemma:derivatives} and Eq.(\ref{eqapp:tildealphachange}) in the last equality.\\
Now, insert Eq.(\ref{eqapp:11}) into Eq.(\ref{eqapp:bdiffinterm}), and apply Lemma~\ref{lemma:derivatives} and the induction hypothesis, then we have
\begin{eqnarray}
\left\| \rvb^{(l-1)} - \rvb_0^{(l-1)}\right\| &\le&  \frac{1}{\sqrt{m}}R \mathsf{L}_{\sigma}^{L-l+1}(c_0+R/\sqrt{m})^{L-l+1} + \frac{1}{\sqrt{m}}c_0\sqrt{m} \left\|\left[\Sigma^{\prime(l)} - \Sigma_0^{\prime(l)}\right]\rvb_0^{(l)} \right\|\nonumber\\
& & + c_0 \mathsf{L}_{\sigma} \left\|\rvb^{(l)} -  \rvb_0^{(l)}\right\|= \tilde{O}\left(\frac{1}{\sqrt{m}}\right).
\end{eqnarray}
Thus, Eq.(\ref{eqapp:inductionb}) holds for $l-1$, and the proof of the induction step is complete. Therefore, by the principle of induction, Eq.(\ref{eqapp:inductionb}) holds for all $l\in [L]$.\\
Now, let's consider $\|\rvb^{(l)}\|_{\infty}$. 
By Lemma~\ref{lemma:everyderivatives} and Lemma~\ref{lemma:derivatives}, with probability at least $1-me^{-c_b^{(l)}\ln^2(m)}$,
\begin{eqnarray}
 \|\rvb^{(l)}\|_{\infty} &\le& \|\rvb^{(l)}_0\|_{\infty} + \|\rvb^{(l)}-\rvb^{(l)}_0\|_{\infty}\nonumber\\
 &\le&\|\rvb^{(l)}_0\|_{\infty} + \|\rvb^{(l)}-\rvb^{(l)}_0\|= \tilde{O}\left(\frac{1}{\sqrt{m}}\right).
\end{eqnarray}
Using union bound, for all $l\in[L]$, we have with probability $1-m\sum_{l=1}^Le^{-c_b^{(l)}\ln^2(m)}$,
\begin{equation}
     \|\rvb^{(l)}\|_{\infty}= \| \frac{\partial f}{\partial \alpha^{(l)}}\|_\infty = \tilde{O}\left(\frac{1}{\sqrt{m}}\right).
\end{equation}
 \end{proof}

 \section{Generalization to other architectures}\label{secapp:cnn_resnet} 
 In this section, we apply Theorem~\ref{thm:hessian_infinity} to  both convolutional neural networks (CNN) and residual networks (ResNets), and show that they both have small Hessian spectral norms when the network width $m$ is sufficiently large and last layer is of linear form. 
 \subsection{Convolutional Neural Networks}\label{secapp:cnn}
 A convolutional neural network (CNN) is a network of the type in Eq.(\ref{eq:generalnn}), with each convolutional layer function $\phi_l$ defined as
 \begin{equation}\label{eq:cnnlayer}
     \alpha^{(l)} = \phi_l(\rmW^{(l)};\alpha^{(l-1)}) = \sigma\left(\frac{1}{\sqrt{m_l}}\rmW^{(l)} \ast \alpha^{(l-1)}\right), \ \forall l\in[L],
 \end{equation}
 where $\ast$ is the convolution operator (see the definition below), and the layer width $m_l = m$ for all $l = 2,3,\cdots,L$, and $m_1=d$ with $d$ as the number of channels of the input. 
 
 To simplify the notation, we consider a one-dimensional CNN, i.e., a ``image'' is an $1$-D array of ``pixels'', and one will find that the analysis in this section also applies to higher dimensional CNNs. We also drop the layer indices $l$, wherever there is no ambiguity. 
 
 We  denote the number of channels for each hidden layer as $m$, the number of pixels in the  ``image'' as $Q$ and the size of each filter as $K$. Furthermore, we use $i,j\in [m]$ as indices of the channels, $q\in [Q]$ as indices of pixels and $k\in [K]$ as indices within the filter.
 The input $\alpha\in \mathbb{R}^{m\times Q}$ is a  matrix, with $m$ rows as channels and $Q$ columns as pixels. The parameters $\rmW\in \mathbb{R}^{K\times m\times m}$ is a order $3$ tensor. The output of the layer function $\phi$ is of size $m\times Q$. In this $1$-D CNN case, the convolution operator is defined as
 \begin{equation}\label{eq:convolution}
     (\rmW \ast \alpha)_{i,q} = \sum_{k=1}^K \sum_{j=1}^m W_{k,i,j} \alpha_{j,q+k-\frac{K+1}{2}}.
 \end{equation}
 
 {\bf Reformulation of  convolutional layer.} Now, we reformulate the convolutional layer function in Eq.(\ref{eq:cnnlayer}) into a fully-connected-like function. Then, we can use the techniques developed in Section~\ref{secapp:fcn_quantities} to prove for the CNN. Specifically, for all $k\in [K]$, define matrices $W^{[k]}$ and $\alpha^{[k]}$ such that each entry $(W^{[k]})_{ij} = W_{k,i,j}$ and $(\alpha^{[k]})_{jq} = \alpha_{j,q+k-\frac{K+1}{2}}$. Then, the convolution operator in Eq.(\ref{eq:convolution}) can be rewritten as
 \begin{equation}
     (\rmW \ast \alpha) = \sum_{k=1}^K W^{[k]}\alpha^{[k]}.
 \end{equation}
 Here in the summation, it is matrix multiplication. Note that, while $W^{[k]}$ are independent from each other for different $k\in[K]$, the inputs $\alpha^{[k]}$ are not independent from each other; instead, they share pixels:  $(\alpha^{[k]})_{j,q} = (\alpha^{[k']})_{j,q+k-k'}$, i.e., each $\alpha^{[k]}$ is a pixel-shifted version of $\alpha$ (newly generated pixels after shift is filled with zeros). 
 
 Therefore, the convolutional layer function can also be written as (for $l>1$)
 \begin{equation}
     \phi\triangleq\phi(\rmW;\alpha) = \sigma \left(\sum_{k=1}^K\frac{1}{\sqrt{m}} W^{[k]}\alpha^{[k]}\right) \triangleq \sigma(\tilde{\alpha}).
 \end{equation}
 Here, we can see 
 we will use this expression of convolutional layer function for analysis in this section. 
 
 Before proceeding to the proof for CNN, we first point out a few useful facts, as summarized in the following lemmas.
 \begin{lemma}\label{lemma:frobenius}
  Given matrices $A,B$ and $C$ such that $A = BC$, we have $\|A\|_F \le \|B\| \|C\|_F$, where $\|B\|$ is the spectral norm of matrix $B$.
 \end{lemma}
 See the proof in Appendix~\ref{secapp:frob_pf}. The following two lemmas provide bounds on the spectral norm of weights and Frobenius norm of hidden layers. These two lemmas (and the proofs) are analogous to Lemma~\ref{lemma:wl2} and~\ref{lemma:activations}, and we omit the proof.
 \begin{lemma}\label{lemma:cnn_wl2}
  Suppose the parameters are initialized as $(W^{[k]}_{0})_{i,j}\sim \mathcal{N}(0,1)$, for all $k\in [K],i,j\in[m]$. Then, with high probability of the random initialization, we have for any $\rmW \in B(\rmW_0,R)$ the following holds
  \begin{equation}
      \|W^{[k]}\| = O(\sqrt{m}), \ \forall k \in [K].
  \end{equation}
 \end{lemma}
 \begin{lemma}\label{lemma:cnn_activation}
  Suppose the parameters are initialized as $(W^{[k]}_{0})_{i,j}\sim \mathcal{N}(0,1)$, for all $k\in [K],i,j\in[m]$ and for all layers. Then, with high probability of the random initialization, we have for any $\rmW \in B(\rmW_0,R)$ the following holds at all hidden layers
  \begin{equation}
      \|\alpha\|_F = O(\sqrt{m}).
  \end{equation}
  And at the input layer,
    \begin{equation}
      \|\alpha\|_F = O(1).
  \end{equation}
 \end{lemma}
 We note that the proof for CNNs is basically analogous to that for fully connected neural networks (FCNs). Here, we refer readers to follow the proof idea for FCNs and  only discuss the main differences below.
 In the following, we focus on analyzing the layers for $l>1$. For the case of $l=1$, we omit the proof, and refer the readers to the discussion in Section~\ref{secapp:fcn_quantities}, which also applies here.
 \paragraph{Matrix spectral norm $\left\|{\partial \phi}/{\partial \rvw}\right\| = O(1)$ and Lipschitz continuity of $\phi$ w.r.t $\alpha$.}
As seen in Section~\ref{secapp:fcn_quantities}, it suffices to prove the boundedness of the operator norms: $\|{\partial \phi}/{\partial \rvw}\|$ and $\|{\partial \phi}/{\partial \alpha}\|$.
 Note that, in the convolutional layer function, the vector of parameters $\rvw$ is reshaped to $\rmW\in\mathbb{R}^{K\times m\times m}$, and the input is reshaped to $\alpha \in \mathbb{R}^{m\times Q}$. Then, the Euclidean norm of the input becomes Frobenius norm $\|\alpha\|_F$, and the Euclidean norm  $\|\rvw\| = (\sum_{k=1}^K \|W^{[k]}\|_F^2)^{1/2}$. 
 
 Then, the spectral norm square
 \begin{eqnarray*}
 \|{\partial \phi}/{\partial \alpha}\|^2 &=& \frac{1}{{m}}\sup_{\|V\|_F=1} \sum_{i=1}^m\sum_{q=1}^Q (\sigma'(\tilde{\alpha}_{i,q}))^2 \Big(\sum_{k=1}^KW^{[k]}V^{[k]}\Big)_{i,q}^2\\
 &\le& \frac{1}{{m}}\mathsf{L}_{\sigma}^2 \sup_{\|V\|_F=1} \Big\|\sum_{k=1}^KW^{[k]}V^{[k]} \Big\|_F^2\\
 &\le& \frac{1}{{m}}\mathsf{L}_{\sigma}^2 \sup_{\|V\|_F=1} \left(\sum_{k=1}^K\|W^{[k]}\| \|V^{[k]} \|_F\right)^2\\
 &\le& \frac{1}{{m}}\mathsf{L}_{\sigma}^2 \left(\sum_{k=1}^K \|W^{[k]}\|\right)^2=O(1).
 \end{eqnarray*}
 Here, in the second inequality, we used Lemma~\ref{lemma:frobenius}, and in the 
 last equality, we used Lemma~\ref{lemma:cnn_wl2}.
 Similarly, using Lemma~\ref{lemma:frobenius} and \ref{lemma:cnn_activation}, we also have
 \begin{eqnarray*}
 \|{\partial \phi}/{\partial \rvw}\|^2 &=& \frac{1}{m} \sup \left\{\sum_{i=1}^m\sum_{q=1}^Q (\sigma'(\tilde{\alpha}_{i,q}))^2 \Big(\sum_{k=1}^KV^{[k]}\alpha^{[k]}\Big)_{i,q}^2: \sum_{k=1}^K\|V^{[k]}\|_F^2 =1\right\}\\
 &\le& \frac{1}{{m}}\mathsf{L}_{\sigma}^2 \sup \left\{ \left\|\sum_{k=1}^KV^{[k]}\alpha^{[k]}\right\|_F^2: \sum_{k=1}^K\|V^{[k]}\|_F^2 =1\right\}\\
 &\le&\frac{1}{{m}}\mathsf{L}_{\sigma}^2 \sup \left\{ \left(\sum_{k=1}^K\|V^{[k]}\|\|\alpha^{[k]}\|_F\right)^2: \sum_{k=1}^K\|V^{[k]}\|_F^2 =1\right\}\\
 &\le& \frac{1}{{m}}\mathsf{L}_{\sigma}^2\left(\sum_{k=1}^K \|\alpha^{[k]}\|_F\right)^2=O(1).
 \end{eqnarray*}

 \paragraph{$(2,2,1)$-norms of order $3$ tensors are $O(1)$.}  
 Recall that the vector of parameters $\rvw$ is reshaped to $\rmW\in\mathbb{R}^{K\times m\times m}$, and the input is reshaped to $\alpha \in \mathbb{R}^{m\times Q}$. Then, by Lemma~\ref{lemma:frobenius} and \ref{lemma:cnn_wl2}, we have
 \begin{align*}
  &\left\|\frac{\partial^2 \phi}{\partial \alpha^2}\right\|_{2,2,1} \\
  &=\sup \left\{ \sum_{i=1}^m\sum_{q=1}^Q \frac{1}{m}\left| \sigma''(\tilde{\alpha}_{i,q}) \left(\sum_{k=1}^K W^{[k]}V_1^{[k]}\right)_{i,q}\left(\sum_{k=1}^K W^{[k]}V_2^{[k]}\right)_{i,q}\right|: \|V_1\|_F=\|V_2\|_F=1\right\}\\
  &\le \frac{\beta_{\sigma}}{2m}\sup\left\{ \left\|\sum_{k=1}^K W^{[k]}V_1^{[k]}\right\|^2_F+\left\|\sum_{k=1}^K W^{[k]}V_2^{[k]}\right\|^2_F: \|V_1\|_F=\|V_2\|_F=1\right\}\\
  &\le \frac{\beta_{\sigma}}{2m} \cdot 2 \left(\sum_{k=1}^K \|W^{[k]}\|\right)^2\\
  &= O(1).
 \end{align*}
 Similarly, by using Lemma~\ref{lemma:frobenius}, \ref{lemma:cnn_wl2} and \ref{lemma:cnn_activation}, we also have
  \begin{align*}
  &\left\|\frac{\partial^2 \phi}{\partial \alpha \partial \rvw}\right\|_{2,2,1} \\
  &=\sup \left\{ \sum_{i=1}^m\sum_{q=1}^Q \frac{1}{m}\left| \sigma''(\tilde{\alpha}_{i,q}) \left(\sum_{k=1}^K V_1^{[k]}\alpha^{[k]}\right)_{i,q}\left(\sum_{k=1}^K W^{[k]}V_2^{[k]}\right)_{i,q}\right|: \sum_{k=1}^K\|V_1^{[k]}\|_F^2= \|V_2\|_F^2 =1\right\}\\
  &\le \frac{\beta_{\sigma}}{2m}\sup\left\{ \left\|\sum_{k=1}^K V_1^{[k]}\alpha^{[k]}\right\|^2_F+\left\|\sum_{k=1}^K W^{[k]}V_2^{[k]}\right\|^2_F: \sum_{k=1}^K\|V_1^{[k]}\|_F^2= \|V_2\|_F^2 =1\right\}\\
  &\le \frac{\beta_{\sigma}}{2m} \cdot  \left(\Big(\sum_{k=1}^K \|\alpha^{[k]}\|_F\Big)^2 + \Big(\sum_{k=1}^K \|W^{[k]}\|\Big)^2 \right)\\
  &= O(1),
 \end{align*}
 and
  \begin{align*}
  &\left\|\frac{\partial^2 \phi}{\partial \rvw^2}\right\|_{2,2,1} \\
  &=\sup \left\{ \sum_{i=1}^m\sum_{q=1}^Q \frac{1}{m}\left| \sigma''(\tilde{\alpha}_{i,q}) \left(\sum_{k=1}^K V_1^{[k]}\alpha^{[k]}\right)_{i,q}\left(\sum_{k=1}^K V_2^{[k]}\alpha^{[k]}\right)_{i,q}\right|: \sum_{k=1}^K\|V_1^{[k]}\|_F^2= \sum_{k=1}^K\|V_2^{[k]}\|_F^2 =1\right\}\\
  &\le \frac{\beta_{\sigma}}{2m}\sup\left\{ \left\|\sum_{k=1}^K V_1^{[k]}\alpha^{[k]}\right\|^2_F+\left\|\sum_{k=1}^K V_2^{[k]}\alpha^{[k]}\right\|^2_F: \sum_{k=1}^K\|V_1^{[k]}\|_F^2= \sum_{k=1}^K\|V_2^{[k]}\|_F^2 =1\right\}\\
  &\le \frac{\beta_{\sigma}}{2m} \cdot  2\Big(\sum_{k=1}^K \|\alpha^{[k]}\|_F\Big)^2\\
  &=O(1).
 \end{align*}
 
 \paragraph{Vector $\infty$-norm is $\tilde{O}(1/\sqrt{m})$.} The proof idea is similar to the case of fully connected case, as in Section~\ref{secapp:infinity-norm-fcn}, i.e., proving by induction. The base case of the induction is the same as fully connected case, and we omit it here. The inductive hypothesis for CNN is: $\max_{i\in[m],q\in [Q]}(\partial f/\partial \alpha^{(l+1)})_{i,q} = \tilde{O}(1/\sqrt{m})$. 
 
 Now, for $l$-th layer, we have
 \begin{align*}
 (\partial f/\partial \alpha^{(l)})_{i,q} &= \sum_{j=1}^m\sum_{q'=1}^Q\sum_{k=1}^K (\partial f/\partial \alpha^{(l+1)})_{j,q'}\sigma'(\tilde{\alpha}^{(l+1)}_{j,q'})\frac{1}{\sqrt{m}}W^{[k]}_{ji}\mathbb{I}_{q=q'-k+\frac{K+1}{2}}\\
 &=\sum_{k=1}^K\sum_{j=1}^m(\partial f/\partial \alpha^{(l+1)})_{j,q+k-\frac{K+1}{2}}\sigma'\Big(\tilde{\alpha}^{(l+1)}_{j,q+k-\frac{K+1}{2}}\Big)\frac{1}{\sqrt{m}}W^{[k]}_{ji}\\
 &\triangleq \sum_{k=1}^K (\partial f/\partial \alpha^{(l)})_{i,q}^{[k]}.
 \end{align*}
 By the same argument as in Section~\ref{secapp:infinity-norm-fcn}, we have: $\max_{i\in[m],q\in [Q]}(\partial f/\partial \alpha^{(l)})_{i,q}^{[k]}= \tilde{O}(1/\sqrt{m})$, for each $k\in[K]$, with high probability of the random initialization. Since $K$ is finite, then we have $\max_{i\in[m],q\in [Q]}(\partial f/\partial \alpha^{(l)})_{i,q} = \tilde{O}(1/\sqrt{m})$ with high probability of the random initialization.

\subsection{Residual Networks (ResNet)}\label{secapp:resnet}
In this subsection we prove that the Hessian spectral norm for ResNet also scales as $\tilde{O}(1/\sqrt{m})$, with $m$ being the width of the network.
We define the ResNet $f$ as follows:
\begin{align}\label{eq:resnet}
  &\alpha^{(1)} = \sigma(\frac{1}{\sqrt{d}}W^{(1)}\rvx), \nonumber\\
 &\alpha^{(l)}= \sigma(\tilde{\alpha}_{res}^{(l)})+\alpha^{(l-1)},  \ \tilde{\alpha}_{res}^{(l)} = \frac{1}{\sqrt{m}}W^{(l)}\alpha^{(l-1)},\ \forall 2 \leq l\leq L, \nonumber \\
 &f = \frac{1}{\sqrt{m}}\rvv^T \alpha^{(L)}.
\end{align}
The parameters $\rmW := \{W^{(1)},W^{(2)},\cdots,W^{(L)},W^{(L+1)}:=\rvv\}$ are initialized following the random Gaussian initialization strategy, i.e., $(W^{(l)}_0)_{ij}\sim \mathcal{N}(0,1), \forall l\in [L]$, and $v_{0,i} \sim \mathcal{N}(0,1)$, $i,j\in [m]$.
\begin{remark}
This definition of ResNet differs from the standard ResNet architecture in~\cite{he2016deep}
that the skip connections are at every layer, instead of every two layers. One will find that the same analysis can be easily generalized to cases where skip connections are at every two or more layer.  The same definition, up to a scaling factor, was also theoretically studied in \cite{du2018gradientdeep}.
\end{remark}

We see that the ResNet is the same as a fully connected neural network, Eq.~(\ref{eq:fcnet}), except that the activations $\alpha^{(l)}$ has an extra additive term $\alpha^{(l-1)}$ from the previous layer, interpreted as skip connection. Because of this similarity, the proof for ResNet is almost identical to that for fully connected networks.
In the following, we sketch the proof for ResNet. Specifically, we focus on the arguments that are new to ResNet, and omit those identical to the fully connected case.

Parallel to Lemma~\ref{lemma:wl2} and \ref{lemma:activations} for fully connected case, we have the following lemmas for the ResNet.
\begin{lemma}\label{lemma:res_wl2}
  Suppose the parameters are initialized as $(W^{(l)}_{0})_{i,j}\sim \mathcal{N}(0,1)$, for all $l\in [L]$, and $v_{0,i} \sim \mathcal{N}(0,1)$, $i,j\in[m]$. Then, with high probability of the random initialization, we have for any $\rmW \in B(\rmW_0,R)$ the following holds
  \begin{equation}
      \|W^{(l)}\| = O(\sqrt{m}), \ \forall l \in [L+1].
  \end{equation}
 \end{lemma}
 \begin{lemma}\label{lemma:res_activation}
  Suppose the parameters are initialized as $(W^{(l)}_{0})_{i,j}\sim \mathcal{N}(0,1)$, for all $l\in [L],$ and $v_{0,i} \sim \mathcal{N}(0,1)$, $i,j\in[m]$. Then, with high probability of the random initialization, we have for any $\rmW \in B(\rmW_0,R)$ the following holds at all hidden layers
  \begin{equation}
      \|\alpha^{(l)}\| = O(\sqrt{m}).
  \end{equation}
  Particularly, for the input layer
    \begin{equation}
      \|\alpha^{(0)}\| =\|\rvx\| =  O(1).
  \end{equation}
 \end{lemma}
 The proofs of the above two lemmas are almost identical to those of Lemma~\ref{lemma:wl2} and \ref{lemma:activations}. We omit the proofs here, and refer interested readers to proofs of Lemma~\ref{lemma:wl2} and \ref{lemma:activations}.

\paragraph{Matrix spectral norm $\left\|{\partial \alpha^{(l)}}/{\partial \rvw^{(l)}}\right\| = O(1)$ and Lipschitz continuity of $\alpha^{(l)}$ w.r.t $\alpha^{(l-1)}$.}

{\bf When} $2 \leq l \leq L$, from the definition of ResNet, Eq.(\ref{eq:resnet}), a ResNet layer $\alpha^{(l)}$ is defined by:
\begin{align}\label{eq:resphi}
   \alpha^{(l)} = \phi_l(W^{(l)};\alpha^{(l-1)}) = \sigma\left(\frac{1}{\sqrt{m}}W^{(l)}\alpha^{(l-1)}\right)+ \alpha^{(l-1)}.
\end{align}
Therefore, we have
\begin{align*}
    \| \partial \alpha^{(l)}/\partial \alpha^{(l-1)}\| &= \sup_{\|\rvv\|=1}\|( \frac{1}{\sqrt{m}}{\Sigma'}^{(l)}W^{(l)}+I)\rvv\| \\
    &\leq \sup_{\|\rvv\|=1}( \frac{1}{\sqrt{m}}\|{\Sigma'}^{(l)}\|\|W^{(l)}\|\|\rvv\|+\|\rvv\|) \\
    &\leq \mathsf{L}_\sigma(c_0+R/\sqrt{m})+1 = O(1).
\end{align*}
We note that $\|\partial \alpha^{(l)}/\partial \rvw^{(l)}\|$ has the same expression as the one of the fully connected networks. By the same argument in Section~\ref{secapp:lipschitz}, as well as Lemma~\ref{lemma:res_activation}, we have $\|\partial \alpha^{(l)}/\partial \rvw^{(l)}\| = O(1)$.

{\bf When} $l=1$, the layer function is defined by
\begin{align*}
    \alpha^{(1)} = \phi_1(W^{(1)};\alpha^{(0)}) = \sigma\left(\frac{1}{\sqrt{d}}W^{(1)}\rvx\right).
\end{align*}
In this layer, the input $\rvx$ is fixed (independent of trainable parameters) and not a dynamical variable. Hence, $\partial \alpha^{(1)}/\partial \alpha^{(0)}$ is not an interesting object in our Hessian analysis.

And we have 
\begin{align*}
    \|\partial \alpha^{(1)}/\partial \rvw^{(1)}\| \leq \frac{1}{\sqrt{d}}\|{\Sigma'}^{(l)}\|\|\rvx\| \leq \mathsf{L}_\sigma C_\rvx = O(1).
\end{align*}

We see that both $\|{\nabla_{\alpha}\phi_l}\|$ and $\|\nabla_\rvw \phi_l\|$ are bounded, hence, the (vector-valued) layer function of ResNet is Lipschitz continuous.

\paragraph{$(2,2,1)$-norms of order $3$ tensors are $O(1)$.} Note that the skip connection term $\alpha^{(l-1)}$ in Eq.(\ref{eq:resphi}) is linear in $\alpha^{(l-1)}$ and independent from $W^{(l)}$. Hence, the order $3$ tensors are exactly the same as in the case of fully connected networks. Applying the same argument as in Section~\ref{secapp:221-norms-fcn} gives the following:
\begin{equation}
    \left\|\frac{\partial^2 \phi_{l}}{(\partial \alpha^{(l-1)})^2}\right\|_{2,2,1} =O(1), \ \left\|\frac{\partial^2 \phi_{l}}{\partial \alpha^{(l-1)}\partial W^{(l)}}\right\|_{2,2,1} = O(1),\  \left\|\frac{\partial^2 \phi_{l}}{(\partial W^{(l)})^2}\right\|_{2,2,1} = O(1).
\end{equation}

\paragraph{Vector $\infty$-norm is $\tilde{O}(1/\sqrt{m})$.}
For a ResNet, define vector $\rvb^{(l)}_{res}:= \partial f /\partial {\alpha^{(l)}}$ for $l\in[L]$. Specifically, $\rvb^{(l)}_{res}$ takes the following form:
\begin{align}
    \rvb^{(l)}_{res} = \prod_{l'=l+1}^L\left(\frac{1}{\sqrt{m}}(W^{(l')})^T{\Sigma'}^{(l')}+I\right)\frac{1}{\sqrt{m}}\rvv.
\end{align}
Compared to the expression of $\rvb^{(l)}$, in Eq.(\ref{eq:derivatives_expr}), which is the fully connected network case, the only difference is that $\rvb^{(l)}_{res}$ for ResNet has an extra additive identity matrix. We argue that the $\infty$-norm $\|\rvb^{(l)}_{res}\|_{\infty}$ is still the order of $\tilde{O}(1/\sqrt{m})$. We show this by induction. 

First, recall that by the analysis in Section~\ref{secapp:infinity-norm-fcn} we have $\|\rvb^{(l)}\|_{\infty}= \tilde{O}(\frac{1}{\sqrt{m}})$ for all $l\in [L]$.

In the base case, $\rvb^{(L)}_{res} =  \frac{1}{\sqrt{m}}\rvv =   \rvb^{(L)}$. Then $\|\rvb^{(L)}_{res}\|_{\infty}= \tilde{O}(\frac{1}{\sqrt{m}})$ holds.

Now, suppose $\|\rvb^{(l+1)}_{res}\|_{\infty}= \tilde{O}(\frac{1}{\sqrt{m}})$ holds. For $\rvb^{(l)}_{res}$, we have
\begin{equation}
    \rvb^{(l)}_{res} = \left(\frac{1}{\sqrt{m}}(W^{(l+1)})^T{\Sigma'}^{(l+1)}+I\right)\rvb^{(l+1)}_{res} = \frac{1}{\sqrt{m}}(W^{(l+1)})^T{\Sigma'}^{(l+1)}\rvb^{(l+1)}_{res} + \rvb^{(l+1)}_{res}.
\end{equation}
By an analogous analysis as in Section~\ref{secapp:infinity-norm-fcn} and \ref{secapp:everyderivatives}, we have that ${\infty}$-norm of the first term is of the order $\tilde{O}(1/\sqrt{m})$. Since $\|\rvb^{(l+1)}_{res}\|_{\infty}$ is also of the order $\tilde{O}(1/\sqrt{m})$, we conclude that $\|\rvb^{(l)}_{res}\|_{\infty}$ is of the order $\tilde{O}(1/\sqrt{m})$.

\subsection{Architecture with mixed layer types}\label{secapp:mix}
So far, we have seen that fully connected networks (Theorem~\ref{thm:fcn}), CNNs(Section~\ref{secapp:cnn}) and ResNets (Section~\ref{secapp:resnet})  have a Hessian spectral norm of order $\tilde{O}(1/\sqrt{m})$. In fact,  our analysis generalizes to architectures with mixed layer types. Note that the analysis for the $(2,2,1)$-norms for order $3$ tensors and the spectral norms of first-derivatives of hidden layers in the proof of Lemma~\ref{lemma:fcn_quantities}, and its counterparts for CNN and ResNet, 
is purely layer-wise, and does not depend on the types of other layers. As for the $\infty$-norm, our analysis is inductive: $\|\nabla_{\alpha^{(l)}}f\|_{\infty} = \tilde{O}(1/\sqrt{m})$ only relies on the structure of the current layer and the fact that $\|\nabla_{\alpha^{(l+1)}}f\|_{\infty} = \tilde{O}(1/\sqrt{m})$. 

\section{Proof of Theorem \ref{thm:lowerbound_H}}\label{secapp:proof_lowerbound_H}
First, we give a useful lemma below:
\begin{lemma}\label{lemma:lower_bound}
Let $\rvx = ( x_1,x_2,...,x_m )$ and $\rvy = ( y_1,y_2,...,y_m ) $ where  $x_1,x_2,...x_m$ and $y_1,y_2,...,y_m$ are i.i.d. random variables with $x_i,y_i \sim \mathcal{N}(0,1)$ for $i\in[m]$. Given an arbitrary radius $R>0$,  for any $\bar{\rvx},\bar{\rvy} \in \mathbb{R}^m$ such that $\|\rvx - \bar{\rvx}\| \leq R$ and $\|\rvy - \bar{\rvy}\| \leq R$  and for any $\delta_1\in(0,1)$, we have, with probability at least $1-2\delta_1$,
\begin{equation}
  \left| \sum_{i=1}^m \bar{x}_i \bar{y}_i^2 \right|\geq \sqrt{3}C_1\delta_1^3 \sqrt{m} -3\log(4/\delta_1) R^4 - R^3,
\end{equation}
for some constant $C_1>0$.
\end{lemma}
The proof of the lemma is deferred to Appendix \ref{secapp:proof_lower_bound}.
\begin{proof}[Proof of Theorem \ref{thm:lowerbound_H}]
Consider an arbitrary parameter setting $\rmW\in B(\rmW_0,R)$.

Note that spectral norm of a matrix is lower bounded by the norm of its blocks, then  
\begin{equation}
  \left\|H(\rmW) \right\| \geq  \left\|\frac{\partial^2 f}{\left(\partial \rvw^{(2)}\right)^2} \right\|.  
\end{equation}
Hence, it's sufficient to lower bound the norm of the Hessian block ${\partial^2 f}/{\left(\partial \rvw^{(2)}\right)^2}$.

With simple computation, the gradient of $f$ w.r.t. $\rvw^{(2)}_i$ is:
\begin{equation}
    \frac{\partial f}{\partial \rvw_i^{(2)}} = \frac{1}{m}\sum_{j=1}^m\rvw_j^{(4)} \sigma'(\tilde{\alpha}_j^{(3)})\rvw_j^{(3)}\alpha_i^{(1)},
\end{equation}
and each entry of the Hessian matrix takes the form:
\begin{eqnarray}
    \frac{\partial^2 f}{\partial \rvw_i^{(2)}\partial \rvw_k^{(2)}} &=& \frac{1}{m^{3/2}}\sum_{j=1}^m\rvw_j^{(4)} \sigma''(\tilde{\alpha}_j^{(3)})(\rvw_j^{(3)})^2\alpha_i^{(1)}\alpha_k^{(1)}\nonumber\\
    &=& \frac{1}{m^{3/2}}\sum_{j=1}^m\rvw_j^{(4)} (\rvw_j^{(3)})^2\alpha_i^{(1)}\alpha_k^{(1)}.
\end{eqnarray}
In the second equality above, we have used $\sigma(x) = \frac{1}{2}x^2$.

Then, the spectral norm of this Hessian block is
\begin{equation}
  \left\|\frac{\partial^2 f}{\left(\partial \rvw^{(2)}\right)^2} \right\|    = \frac{1}{m^{3/2}}\left|\sum_{j=1}^m\rvw_j^{(4)} (\rvw_j^{(3)})^2\right| \left\|\alpha^{(1)}\right\|^2.
\end{equation}

For the last factor $\|\alpha^{(1)}\|^2$, using the tail bound for $\chi^2(m)$~\cite{laurent2000adaptive}, we have, with probability at least $1-e^{-m/16}$,
\begin{align*}
    \|\alpha^{(1)}\|^2 = \frac{x^2}{4} \sum_{i=1}^m \left(\rvw_i^{(1)}\right)^2 \geq \frac{x^2}{4}\left(\frac{m}{2} -R^2\right).
\end{align*}

Applying Lemma~\ref{lemma:lower_bound} to the factor $\left|\sum_{j=1}^m\rvw_j^{(4)} (\rvw_j^{(3)})^2\right|$ and using union bound,  for an arbitrary $\delta\in(0,1)$, we have,  with probability at least $1-2\delta-e^{-m/16}$,
\begin{equation}
     \left\|\frac{\partial^2 f}{\left(\partial \rvw^{(2)}\right)^2} \right\| \geq \frac{x^2}{4}\left(\frac{1}{2} -\frac{R^2}{m}\right)\left(\sqrt{3}C_1\delta^3  -\frac{3\log(4/\delta) R^4 + R^3}{\sqrt{m}}\right).
\end{equation}
Hence, we get the lower bound of the Hessian spectral norm at $\rmW\in B(\rmW_0,R)$
\begin{align}
    \left\|H(\rmW) \right\| \geq  \left\|\frac{\partial^2 f}{\left(\partial \rvw^{(2)}\right)^2} \right\| \geq \frac{x^2}{4}\left(\frac{1}{2} -\frac{R^2}{m}\right)\left(\sqrt{3}C_1\delta^3  -\frac{3\log(4/\delta) R^4 + R^3}{\sqrt{m}}\right).
\end{align}

\end{proof}

\section{Proofs of Technical Lemmas}\label{secapp:technical}
\subsection{Proof of Lemma~\ref{lemma:combine}}\label{secapp:techincal1}
\begin{proof}
\begin{align*}
\| H \| &= \left\|\left(\begin{array}{cccc}
    H^{(1,1)} & 0 & \cdots & 0\\
    0 & 0 & \cdots & 0\\
    \vdots & \vdots & \ddots & \vdots \\
    0 & 0 & \cdots & 0
    \end{array}
    \right) + \left(\begin{array}{cccc}
    0 & H^{(1,2)} & \cdots & 0\\
    0 & 0 & \cdots & 0\\
    \vdots & \vdots & \ddots & \vdots \\
    0 &0& \cdots & 0
    \end{array}
    \right)  + \cdots + \left(\begin{array}{cccc}
    0 & 0 & \cdots & 0\\
    0 & 0 & \cdots & 0\\
    \vdots & \vdots & \ddots & \vdots \\
    0 &0& \cdots &H^{(L+1,L+1)}
    \end{array}
    \right)\right\| \\
    &\leq \sum_{l_1,l_2} \| H^{(l_1,l_2)}\|.
\end{align*}
\end{proof}

\subsection{Proofs for Gaussian Random Initialization}\label{secapp:gaussian}
\begin{proof}[Proof of Lemma~\ref{lemma:gaussl2}]
Consider an arbitrary random matrix $W \in \mathbb{R}^{m_1\times m_2}$ with each entry $W_{ij}\sim \mathcal{N}(0,1)$. By Corollary 5.35 of~\cite{vershynin2010introduction}, for any $t>0$, we have with probability at least $1-2\mathrm{exp}(-\frac{t^2}{2})$, 
\begin{equation}
    \| W \|_2 \leq \sqrt{m_1}+ \sqrt{m_2} + t.
\end{equation}
In particular, for the initial parameter setting $\rmW_0$, we have
\begin{align*}
         \| W_0^{(1)} \|_2 &\leq \sqrt{d}+\sqrt{m}+t,  \\
          \| W_0^{(l)} \|_2 &\leq 2\sqrt{m}+t, \quad l \in \{2,3,...,L\},\\
         \| W_0^{(L+1)} \|_2 &\leq \sqrt{m}+1+t.
\end{align*}
Letting $t = \sqrt{m}$ and noting that $m>d$, we finish the proof.
\end{proof}


\subsection{Proof of Lemma~\ref{lemma:wl2}}\label{secapp:weight}
\begin{proof}
By triangle inequality and the definition $\|\rmW\| = \sum_{l=1}^{L+1} \|W^{(l)}\|_F$, we have for all layers, i.e., $l\in [L+1]$,
\begin{equation}
    \|W^{(l)}\|_2 \le \|W^{(l)}_0\|_2 + \|W^{(l)}-W^{(l)}_0\|_2 \le \|W^{(l)}_0\|_2 + \|W^{(l)}-W^{(l)}_0\|_F \le c_0\sqrt{m} + R.
\end{equation}
Note that, at the output layer, $W^{(L+1)}$ i.e. $\rvv$ is a vector, and the Frobenius norm $\|\cdot\|_F$ reduces to the Euclidean norm $\|\cdot\|$.
\end{proof}
\subsection{Proof of Lemma~\ref{lemma:activations}}\label{secapp:activations}
\begin{proof}
To analyze $\|\alpha^{(l)}(\rmW)\|$, let's first consider the input layer, i.e., $l=0$: $\|{\alpha}^{(0)}\| = \|\rvx\| \le \sqrt{d}\|\rvx\|_{\infty} \le \sqrt{d}C_{\rvx}$, where $d$ is the dimension of the input $\rvx$. Then we prove Eq.(\ref{eqapp:activations}) by induction.\\
For the first hidden layer $l=1$, 
\begin{eqnarray}
 \|{\alpha}^{(1)}(\rmW)\| &=&  \left\|\sigma\left(\frac{1}{\sqrt{d}}W^{(1)}{\alpha}^{(0)}\right)\right\|\nonumber\\
 &\le& \frac{1}{\sqrt{d}}\mathsf{L}_{\sigma}\|W^{(1)}\|_2\|{\alpha}^{(0)}\|+ \sigma(0)\nonumber\\
 &\le&\frac{1}{\sqrt{d}}\mathsf{L}_{\sigma}(c_0\sqrt{m}+R)\|{\alpha}^{(0)}\|+\sigma(0) \nonumber\\
 &\le&\mathsf{L}_{\sigma}(c_0+R/\sqrt{m})\sqrt{m}C_{\rvx}+\sigma(0).\label{eqapp:lis1}
\end{eqnarray}
Above, we used the $\mathsf{L}_{\sigma}$-Lipschitz continuity and applied Lemma~\ref{lemma:wl2} in the second inequality.\\
Now, suppose for $l$-th layer we have
\begin{equation}
\|{\alpha}^{(l)}(\rmW)\| \le \mathsf{L}_{\sigma}^l(c_0+R/\sqrt{m})^{l}{\sqrt{m}}C_{\rvx}+ \sum_{i=1}^l \mathsf{L}_{\sigma}^{i-1}(c_0+R/\sqrt{m})^{i-1}\sigma(0).
\end{equation}
Then, by a similar argument as in Eq.(\ref{eqapp:lis1}), we can get
\begin{align*}
 \|{\alpha}^{(l+1)}(\rmW)\| &= \left\|\sigma\Big(\frac{1}{\sqrt{m}}W^{(l+1)}{\alpha}^{(l)}(\rmW)\Big)\right\|\\
 &\le \mathsf{L}_{\sigma}^{l+1}(c_0+R/\sqrt{m})^{l+1}{\sqrt{m}}C_{\rvx} +\sum_{i=1}^{l+1} \mathsf{L}_{\sigma}^{i-1}(c_0+R/\sqrt{m})^{i-1}\sigma(0).
\end{align*}
\end{proof}
 
\subsection{Proof of Lemma~\ref{lemma:everyactivation}}\label{secapp:everyactivation}
\begin{proof}
When $2 \leq l \leq L$, $|\alpha_i^{(l)}|$ takes the following form:
\begin{align*}
    |\alpha^{(l)}_i| &= \left|\sigma\left(\frac{1}{\sqrt{m}}\sum_{k=1}^m\ W^{(l)}_{ik}\alpha_k^{(l-1)}\right)\right|\\
    &\leq \left|\frac{\mathsf{L}_\sigma}{\sqrt{m}}\sum_{k=1}^m\ W^{(l)}_{ik}\alpha_k^{(l-1)}\right| + |\sigma(0)|,
\end{align*}
where we can see $ \sum_{k=1}^m\ W^{(l)}_{ik}\alpha_k^{(l-1)} \sim \mathcal{N}(0,\|\alpha^{(l-1)}\|^2)$ since $W^{(l)}_{ik}\sim \mathcal{N}(0,1)$ at initialization.\\
By the concentration inequality for Gaussian random variable, we have
\begin{align*}
    \P[|\alpha^{(l)}_i| \geq \ln(m) +|\sigma(0)|] &\leq  \P[\left|\frac{\mathsf{L}_\sigma}{\sqrt{m}}\sum_{k=1}^m\ W^{(l)}_{ik}\alpha_k^{(l-1)}\right|
    \geq \ln(m) ]\\
    &\leq  2e^{-\frac{m\ln^2(m)}{2\mathsf{L}_\sigma^2\|\alpha^{(l-1)}\|^2}}\\
    &= 2e^{-{c_\alpha^{(l)}\ln^2(m)}},
\end{align*}
for  $c^{(l)}_{\alpha} =\frac{m}{2\mathsf{L}_\sigma^2 \| \alpha^{(l-1)}\|^2} =\Omega(1)$ by Lemma~\ref{lemma:activations}.\\ 
When $l=1$, we have
\begin{align*}
    |\alpha_i^{(1)}| &= \left|\sigma\left(\frac{1}{\sqrt{d}}\sum_{k=1}^d\ W^{(1)}_{ik}\rvx_k\right)\right|\\
    &\leq \left|\frac{\mathsf{L}_\sigma}{\sqrt{d}}\sum_{k=1}^d W^{(1)}_{ik}\rvx_k\right| + |\sigma(0)|.
\end{align*}
Similarly, at initialization,  $\sum_{k=1}^d W^{(1)}_{ik}\rvx_k \sim \mathcal{N}(0,\|\rvx\|^2)$. Hence
\begin{align*}
    \P[|\alpha^{(1)}_i| \geq \ln(m) +|\sigma(0)|] &\leq  \P\left[\left|\frac{\mathsf{L}_\sigma}{\sqrt{d}}\sum_{k=1}^d\ W^{(l)}_{ik}\rvx_k\right|
    \geq \ln(m) \right]\\
    &\leq  2e^{-\frac{d\ln^2(m)}{2\mathsf{L}_\sigma^2\|\rvx\|^2}}\\
    &= 2e^{-{c_\alpha^{(0)}\ln^2(m)}},
\end{align*}
where we denote $\frac{d}{\mathsf{L}_\sigma^2\|\rvx\|}$ by $c_\alpha^{(0)}$, which is of the order $\Theta(1)$.\\
Therefore, $|\alpha_i^{(l)}| = \tilde{O}(1)$ with probability at least $1 - 2e^{-{c_\alpha^{(l)}\ln^2(m)}}$ for all $l\in[L]$.
\end{proof} 

\subsection{Proof of Lemma~\ref{lemma:derivatives}}\label{secapp:derivatives}
\begin{proof}
The expression of the derivatives $\rvb^{(l)}$ is 
\begin{equation}
    \rvb^{(l)} = \left(\prod_{l'=l+1}^{L}\frac{1}{\sqrt{m}}(W^{(l')})^T\Sigma^{\prime(l')}\right)\frac{1}{\sqrt{m}}\rvv,
\end{equation}
where $\Sigma^{\prime(l')}$ is a diagonal matrix with  $(\Sigma^{\prime(l')})_{ii} = \sigma'(\tilde{\alpha}^{(l)}_{i}(\rmW))$.\\
We prove the lemma by induction. When $l = L$, using Lemma~\ref{lemma:wl2}, we have
\begin{equation}
\|\rvb^{(L)}\| = \frac{1}{\sqrt{m}}\|\rvv\| \le   \frac{1}{\sqrt{m}}(c_0\sqrt{m}+R) = c_0 + R/\sqrt{m}.
\end{equation}
Suppose at $l$-th layer, $\|\rvb^{(l)}\| \le \mathsf{L}_\sigma^{L-l}(c_0 + R/\sqrt{m})^{L-l+1}$. Then
\begin{align*}
    \| \rvb^{(l-1)} \| &= \| \frac{1}{\sqrt{m}}(W^{(l)})^T\Sigma^{\prime(l)}\rvb^{(l)} \|  \\
    &\leq \frac{1}{\sqrt{m}}\|W^{(l)}\|_2\|\Sigma^{\prime(l)}\|_2\|\rvb^{(l)} \| \\
    &\leq  (c_0 + R/\sqrt{m})\mathsf{L}_\sigma\| \rvb^{(l)} \| \\
    &\leq \mathsf{L}_\sigma^{L-l+1}(c_0 + R/\sqrt{m})^{L-l+2}.
\end{align*}
Above, we used Lemma~\ref{lemma:wl2} and the $\mathsf{L}_{\sigma}$-Lipschitz continuity of the activation function $\sigma(\cdot)$  in the second inequality.\\
Setting $R=0$, we immediately obtain Eq.(\ref{eqapp:rvb0-2norm}).
    
\end{proof}

\subsection{Proof of Lemma~\ref{lemma:everyderivatives}}\label{secapp:everyderivatives}
\begin{proof}
We prove it by induction.
When $l = L$, $\rvb_0^{(L)} = \frac{1}{\sqrt{m}}\rvv_0$. Since $\rvv_{0,i}\sim \mathcal{N}(0,1)$, by the concentration inequality, for every $i\in[m]$, we have
\begin{align*}
    \P[|\rvv_{0,i}| \geq \ln(m)] \leq 2e^{\frac{-\ln^2(m)}{2}}.
\end{align*}
By union bound, with probability at least $1-2me^{\frac{-\ln^2(m)}{2}}$, 
\begin{align*}
    \|\rvv_0\|_\infty \leq \ln(m),
\end{align*}
in other words,
\begin{align*}
    \|\rvb_0^{(L)}\|_\infty = \tilde{O}(1/{\sqrt{m}}).
\end{align*}
Supposing with probability $1-me^{-c^{(l)}_b \ln^2(m)}$ for some constant $c^{(l)}_b>0$,  we have $\|\rvb_0^{(l)}\|_\infty = \tilde{O}(1/\sqrt{m})$. Then we show $\|\rvb_0^{(l-1)}\|_\infty = \tilde{O}(1/\sqrt{m})$ with probability $1-me^{-c^{(l-1)}_b \ln^2(m)}$ for some constant $c^{(l-1)}_b>0$.\\
For simplicity, in the rest of the proof, we hide the  subscript $0$. Hence we denote $\rvb_0^{(l-1)} = \frac{1}{\sqrt{m}}(W^{(l-1)}_0)^T {\Sigma'}_0^{(l-1)} \rvb_0^{(l)}$ by
\begin{align*}
    \rvb^{(l-1)} = \frac{1}{\sqrt{m}}(W^{(l-1)})^T {\Sigma'}^{(l-1)} \rvb^{(l)},
\end{align*}
where $(W^{(l-1)})_{ij} \sim \mathcal{N}(0,1)$.\\
Similarly, we analyze every component of $\rvb^{(l-1)}$:
\begin{align*}
|\rvb_i^{(l-1)}| &= \left|\frac{1}{\sqrt{m}}\sum_{k=1}^m W^{(l-1)}_{ki}\sigma'\left(\frac{1}{\sqrt{m}}\sum_{j=1}^m W^{(l-1)}_{kj}\alpha^{(l-2)}_j\right)\rvb_{k}^{(l)} \right|   \\ 
&\leq \left|\frac{1}{\sqrt{m}}\sum_{k=1}^m W^{(l-1)}_{ki}\sigma'\left(\frac{1}{\sqrt{m}}\sum_{j\neq i}^m W^{(l-1)}_{kj}\alpha^{(l-2)}_j\right)\rvb_{k}^{(l)} + \frac{1}{m}\beta_\sigma \alpha_i^{(l-2)}\sum_{k=1}^m (W^{(l-1)}_{ki})^2 \rvb_k^{(l)}\right| \\
&\leq \left|\frac{1}{\sqrt{m}}\sum_{k=1}^m W^{(l-1)}_{ki}\sigma'\left(\frac{1}{\sqrt{m}}\sum_{j\neq i}^m W^{(l-1)}_{kj}\alpha^{(l-2)}_j\right)\rvb_{k}^{(l)}\right| + \left|\frac{1}{m}\beta_\sigma \alpha_i^{(l-2)}\sum_{k=1}^m (W^{(l-1)}_{ki})^2 \rvb_k^{(l)}\right|.
\end{align*}
For the first term, we use a Gaussian random variable to bound it:
\begin{align*}
    \frac{1}{\sqrt{m}}\sum_{k=1}^m W^{(l-1)}_{ki}\sigma'\left(\frac{1}{\sqrt{m}}\sum_{j\neq i}^m W^{(l-1)}_{kj}\alpha^{(l-2)}_j\right)\rvb_{k}^{(l)} \leq \frac{\mathsf{L}_\sigma}{\sqrt{m}} \sum_{k=1}^m W^{(l-1)}_{ki} \rvb_k^{(l)} \sim \mathcal{N}\left(0, \frac{L^2_\sigma}{m}\|\rvb^{(l)}\|^2\right).
\end{align*}
Using the concentration inequality, we have
\begin{align*}
    \P\left[\left|\frac{\mathsf{L}_\sigma}{\sqrt{m}} \sum_{k=1}^m W^{(l-1)}_{ki} \rvb_k^{(l)}\right| \geq \frac{\ln(m)}{\sqrt{m}}\right] \leq 2e^{-\frac{\ln^2(m) }{{2L^2_\sigma}\|\rvb^{(l)}\|^2}} \leq 2e^{-c_\sigma^{(l)} \ln^2(m)},
\end{align*}
for some $c^{(l)}_\sigma = \frac{1}{2L^2_\sigma \|\rvb^{(l)}\|^2} \geq \frac{1}{2\mathsf{L}_\sigma^{2L-2l+2}{c_0}^{2L-2l+2}}$ by Lemma~\ref{lemma:derivatives}.\\
For the second term,  we have
\begin{align*}
   \frac{1}{m}\beta_\sigma \alpha_i^{(l-2)}\sum_{k=1}^m (W^{(l-1)}_{ki})^2 \rvb_k^{(l)}  \leq \frac{1}{m}\beta_\sigma|\alpha_i^{(l-2)}|\|\rvb^{(l)}\|_\infty \sum_{k=1}^m (W^{(l-1)}_{ki})^2,
\end{align*}
where we can see $\sum_{k=1}^m (W^{(l-1)}_{ki})^2 \sim \chi^2(m)$.\\
By Lemma~\ref{lemma:everyactivation}, with probability $1-e^{-c^{(l-2)}_\alpha \ln^2(m)}$, we get $|\alpha^{(l-2)}_i| = \tilde{O}(1)$. Hence, by Lemma 1 in \cite{laurent2000adaptive}, there exist constants $\tilde{c}_1,\tilde{c}_2,\tilde{c}_3>0$,  such that
\begin{align*}
    \P\left[\frac{1}{m}\beta_\sigma|\alpha_i^{(l-2)}|\|\rvb^{(l)}\|_\infty \sum_{k=1}^m (W^{(l-1)}_{ki})^2 \geq \tilde{c}_1\frac{\ln^{\tilde{c}_3}(m)}{\sqrt{m}}\right] \leq e^{-\tilde{c}_2m}, 
\end{align*}
with probability $1-me^{-c^{(l)}_b \ln^2(m)}$ by the induction hypothesis.\\
Combining these probability terms, there exists a constant $c_b^{(l-1)}$ such that 
\begin{align*}
    e^{-c^{(l-1)}_b \ln^2(m)} \leq  me^{-c^{(l)}_b \ln^2(m)} + 2e^{-c_\sigma^{(l)} \ln^2(m)} + 2e^{-c^{(l-2)}_\alpha \ln^2(m)} +  e^{-\tilde{c}_2m} .
\end{align*}
Then with probability at least $1- e^{-c^{(l-1)}_b \ln^2(m)}$,
\begin{align*}
    |\rvb_i^{(l-1)}| = \tilde{O}(1/\sqrt{m}).
\end{align*}
By union bound, with probability at least $1-me^{-c^{(l-1)}_b \ln^2(m)}$, we have 
\begin{align*}
    \|\rvb^{(l-1)}\|_\infty = \tilde{O}(1/{\sqrt{m}}).
\end{align*}
Hence by the principle of induction, for all $l\in [L]$, with probability at least $1-me^{-c^{(l)}_b \ln^2(m)}$ for some constant $c^{(l)}_b>0$, we have
\begin{align}
    \|\rvb^{(l)}\|_\infty = \tilde{O}(1/{\sqrt{m}}).
\end{align}
\end{proof}

\subsection{Proof of Lemma~\ref{lemma:frobenius}}\label{secapp:frob_pf}
\begin{proof}
Let $A = (\rva_1,\rva_2,\cdots,\rva_d)$ and $C = (\rvc_1,\rvc_2,\cdots,\rvc_d)$, where each $\rva_i$ is a column of the matrix $A$ and each $\rvc_i$ is a column of the matrix $C$. Then we have
\begin{equation}
    \rva_i = B\rvc_i, \ \forall i\in[d].
\end{equation}
Now, for the Frobenius norm, we have
\begin{align*}
    \|A\|_F^2 &= \sum_{i=1}^d \|\rva_i\|^2
    = \sum_{i=1}^d \|B\rvc_i\|^2
    \le \sum_{i=1}^d \|B\|^2 \|\rvc_i\|^2
    = \|B\|^2 \|C\|_F^2.
\end{align*}
Hence, $\|A\|_F\le \|B\|\|C\|_F$.
\end{proof}

\subsection{Proof of Lemma \ref{lemma:lower_bound} }\label{secapp:proof_lower_bound}
\begin{proof}
First, let's write $\bar{\rvx}$ and $\bar{\rvy}$ as $$
    \bar{\rvx} = \rvx + \rvs, \quad \bar{\rvy} = \rvy + \rvt,
$$
where $\rvs := (s_1,
\cdots, s_m) = (\bar{x}_1-x_1, \cdots, \bar{x}_m-x_m)$ and $\rvt := (t_1, \cdots, t_m)=(\bar{y}_1-y_1, \cdots, \bar{y}_m-y_m)$.
By the condition of the Lemma, we have $\|\rvs\| \le R$ and $\|\rvt\| \leq R$.

Then, we have:
\begin{align}\nonumber \label{eq:lower_bound}
    \left| \sum_{i=1}^m \bar{x}_i \bar{y}_i^2 \right| &= \left| \sum_{i=1}^m (x_i+s_i) (y_i+t_i)^2 \right|\nonumber\\
    &= \left|\sum_{i=1}^m (x_i+s_i)y_i^2 + 2 \sum_{i=1}^m x_it_iy_i + 2 \sum_{i=1}^m s_it_iy_i + \sum_{i=1}^m x_it_i^2 + \sum_{i=1}^m s_it_i^2 \right|\nonumber\\ 
    &\geq \left|\sum_{i=1}^m (x_i+s_i)y_i^2 + 2 \sum_{i=1}^m x_it_iy_i\right| - \left|2 \sum_{i=1}^m s_it_iy_i\right|-\left|\sum_{i=1}^m x_it_i^2\right|-\left|\sum_{i=1}^m s_it_i^2\right|.
\end{align}
Now, let's lower bound the first term and upper bound the last three terms.

For the first term, we lower bound it by the anti-concentration inequality, i.e., Theorem 8 in \cite{carbery2001distributional}. Specifically, for any $\delta_1 \in (0,1)$, there exists a constant $C_1>0$, such that, with probability at least $1-\delta_1$,
\begin{equation}
    \left|\sum_{i=1}^m (x_i+s_i)y_i^2 + 2 \sum_{i=1}^m x_it_iy_i\right|\geq C_1\delta_1^3 \sqrt{\mathbb{E}\left[\left(\sum_{i=1}^m (x_i+s_i)y_i^2 + 2 \sum_{i=1}^m x_it_iy_i\right)^2\right]}.\label{eqapp:bound1}
\end{equation}

On the other hand,
\begin{align}
    & {\mathbb{E}}\left[ \left(\sum_{i=1}^m (x_i+s_i)y_i^2 + 2 \sum_{i=1}^m x_it_iy_i\right)^2\right] \nonumber\\
    &= \mathbb{E}\left[\left(\sum_{i=1}^m (x_i+s_i)y_i^2\right)^2\right] + 4\mathbb{E}\left[\left(\sum_{i=1}^m x_it_iy_i\right)^2\right]  + 4\mathbb{E}\left[\left( \sum_{i=1}^m (x_i+s_i)y_i^2\right)\left(\sum_{i=1}^m x_it_iy_i\right)\right] \nonumber\\
    &\geq  \mathbb{E}\left[\left(\sum_{i=1}^m (x_i+s_i)y_i^2\right)^2\right] \nonumber\\
    &= \mathbb{E}\left[\left(\sum_{i=1}^m (x_i+s_i)^2y_i^4\right)\right]+ 2\sum_{i\neq j}s_is_j\nonumber\\
    &= 3m + 3\sum_{i=1}^m s_i^2 + 2\sum_{i\neq j}s_is_j \nonumber\\
    &= 3m + 2\sum_{i=1}^m s_i^2 + (\sum_{i=1}^m s_i)^2 \nonumber\\
    &\geq 3m,\label{eqapp:bound2}
\end{align}
where the expectation $\mathbb{E}[\cdot ]$ is taken over the random variables $x_1,...x_m,,y_1,...,y_m$. Above, we used the fact that these random variables are i.i.d. and $\mathbb{E}[x_i^2] = 1$ and that $\mathbb{E}[y_i^4] = 3$ for all $i\in[m]$.

Combining Eq.(\ref{eqapp:bound1}) and (\ref{eqapp:bound2}), we have, with probability at least $1-\delta_1$,
\begin{equation}
    \left|\sum_{i=1}^m (x_i+s_i)y_i^2 + 2 \sum_{i=1}^m x_it_iy_i\right|\geq \sqrt{3}C_1\delta_1^3 \sqrt{m}.\label{eqapp:bound3}
\end{equation}



For the second and third terms, we notice that
\begin{align*}
 \sum_{i=1}^m s_it_iy_i \sim \mathcal{N}\left(0,\sum_{i=1}^m s_i^2t_i^2\right), \ \sum_{i=1}^m x_it_i^2 \sim \mathcal{N}\left(0,\sum_{i=1}^m t_i^4\right).
\end{align*}
And it's not hard to see
\begin{align*}
    \sum_{i=1}^m s_i^2t_i^2 \leq R^4, \ \sum_{i=1}^m t_i^4\leq R^4.
\end{align*}
By the concentration inequality for Gaussian random variable (Prop 2.1.9 in \cite{tao2012topics}) and union bound, we have, for any $\delta_2 \in (0,1)$, 
\begin{align}\label{eq:tail_bound3}
    \left| \sum_{i=1}^m s_it_iy_i\right| \leq \log(4/\delta_2) R^4, \ \left|\sum_{i=1}^m x_it_i^2\right| \leq \log(4/\delta_2) R^4,
\end{align}
with probability $1-\delta_2$.


As for the last term, using $\|\rvs\|\le R$ and $\|\rvt\|\le R$, it is easy to have the following bound:
\begin{align}\label{eq:constant_bound}
    \left|\sum_{i=1}^m s_it_i^2\right| \leq R, \quad \left|\sum_{i=1}^m t_i^2\right| = R^3.
\end{align}

Putting inequalities Eq.(\ref{eqapp:bound3}), Eq.(\ref{eq:tail_bound3}), and Eq.(\ref{eq:constant_bound}) into Eq.(\ref{eq:lower_bound}) and using union bound, we have with probability at least $1-\delta_1-\delta_2$,
\begin{equation*}
   \left| \sum_{i=1}^m \bar{x}_i \bar{y}_i^2 \right|= \left| \sum_{i=1}^m (x_i+s_i) (y_i+t_i)^2 \right| \geq \sqrt{3}C_1\delta_1^3 \sqrt{m} -3\log(4/\delta_2) R^4 - R^3.
\end{equation*}
If we set $\delta_1=\delta_2$ in the above analysis, we have, with probability at least $1-2\delta_1$,
\begin{equation*}
   \left| \sum_{i=1}^m \bar{x}_i \bar{y}_i^2 \right|= \left| \sum_{i=1}^m (x_i+s_i) (y_i+t_i)^2 \right| \geq \sqrt{3}C_1\delta_1^3 \sqrt{m} -3\log(4/\delta_1) R^4 - R^3.
\end{equation*}
\end{proof}

\end{document}